\documentclass[preprint,12pt,authoryear]{elsarticle}
\usepackage{amsmath}
\usepackage{amsthm}
\usepackage{amssymb}
\usepackage{array}
\usepackage{derivative}
\usepackage{graphicx}
\usepackage{pdflscape}  
\usepackage{xcolor}     
\usepackage{comment}
\usepackage[linesnumbered,ruled,vlined]{algorithm2e}
\usepackage{url}
\usepackage{}
\usepackage[hidelinks]{hyperref}  
\hypersetup{pdfborderstyle={/S/U/W 1}}  
\usepackage[normalem]{ulem} 
\usepackage{multirow}
\newtheorem{theorem}{Theorem}

\newtheorem{corollary}[theorem]{Corollary}

\newtheorem{remark}[theorem]{Remark}
\newtheorem{proposition}[theorem]{Proposition}

\newtheorem*{proposition*}{Proposition}

\newcommand{\rset}{\mathbf{R}}

\newcommand{\sign}{\text{sign}}

\newcommand{\I}{\mathcal{I}}

\newcommand{\A}{\mathcal{A}}

\newcommand{\N}{\mathcal{N}}

\providecommand{\norm}[1]{\left\lVert#1\right\rVert}
\providecommand{\abs}[1]{|#1|}

\newcommand{\bi}{\begin{itemize}}
\newcommand{\ei}{\end{itemize}}
\newcommand{\be}{\begin{equation}}
\newcommand{\ee}{\end{equation}}
\newcommand{\baa}{\left[\begin{array}}
\newcommand{\eaa}{\end{array}\right]}
\newcommand{\bam}{\begin{bmatrix}}
\newcommand{\eam}{\end{bmatrix}}
\newcommand{\bt}{\begin{tabular}}
\newcommand{\et}{\end{tabular}}

\DeclareMathOperator{\spanv}{span}

\allowdisplaybreaks

\begin{document}

\begin{frontmatter}

\title{Fusing Dictionary Learning and Support Vector Machines for Unsupervised Anomaly Detection}

\author[UB]{Paul Irofti}
\author[UB]{Iulian-Andrei Hîji}
\author[UB]{Andrei Pătrașcu}
\author[TUIASI]{Nicolae Cleju}

\affiliation[UB]{
    organization={Research Center for Logic, Optimization and Security (LOS),\\
    Department of Computer Science, Faculty of Mathematics and Computer Science,\\
    University of Bucharest},
    country={Romania}}


\affiliation[TUIASI]{organization={Faculty of Electronics, Telecommunications and Information Technology \\
Technical University of Iași},
            country={Romania}}

\begin{abstract}
We study in this paper the improvement of one-class support vector machines (OC-SVM) through  sparse representation techniques for unsupervised anomaly detection. As Dictionary Learning (DL) became recently a common analysis technique that reveals hidden sparse patterns of data,
our approach uses this insight to endow unsupervised detection with more control on pattern finding and dimensions.  
We introduce a new anomaly detection model that
unifies the OC-SVM and DL residual functions into a single composite objective, subsequently solved through K-SVD-type iterative algorithms.
A closed-form of the alternating K-SVD iteration is explicitly derived for the new composite model and practical implementable schemes are discussed.
The standard DL model is adapted for the Dictionary Pair Learning (DPL) context,
where the usual sparsity constraints are naturally eliminated.
Finally,
we extend both objectives to the more general setting
that allows the use of kernel functions.
The empirical convergence properties of the resulting algorithms are provided and an in-depth analysis of their parametrization is performed while also demonstrating their numerical performance in comparison with existing methods.
\end{abstract}


\end{frontmatter}
\section{Introduction}
\label{sec:intro}

The anomaly detection task seeks to find the needle in the haystack.
Given a large dataset its aim is to
identify the few particular points,
called anomalies or outliers,
that stand out against the commonality of the vast majority
(also called inliers).
The two classes are highly unbalanced
and the ratio of outliers,
called contamination rate,
can often be below one percent.
We are particularly interested in the unsupervised setting
where we are given an unlabeled dataset
based on which our machine learning algorithms have to separate
the two classes.
In order to achieve this goal,
in this paper we study the possibility of connecting the objective
of two well established methods in the field:
dictionary learning and one-class support vector machines.

Dictionary learning (DL)~\citep{dl_book}
is a matrix factorization method
where a redundant base, called dictionary or frame,
is learned such that it provides sparse representations~\citep{Elad_book}
for the given dataset.
Nowadays,
dictionary learning is a well established standard machine learning method
with multiple extensions, variants and specializations
that lead to various applications across fields such as
audio and image processing,
compression,
fault detection and isolation,
computer vision
and classification.

Let the data items in the given dataset be organized
as columns inside a large matrix.
The main DL strength over similar dimensionality-reduction techniques,
such as principal component analysis or low-rank approximations,
is its ability to find a separate sub-space for each data-item in the dataset.
In other words,
the resulting representation columns are sparse and their support is not similar.
In fact it was recently shown that they follow a
distribution mixture composed of multiple multivariate variables following a truncated normal distribution \citep{Xiao23_StatSuppOMP}.

Starting from the column support distribution described above,
we are now interested in selecting only a few lines
with whom we can still adequately represent the inliers
but not the outliers.
The support for each data-item remains sparse
and distinct from the other columns,
it is just limited to the selected lines.
We call this uniform sparse representation.
To this end,
our work in~\citep{IRP22_ClipDL}
performs line selection
by applying trimming regularization techniques
to eliminate one line at a time until a stable support is reached.
After the DL process is completed,
we perform anomaly detection:
we test to see if a (new) data-item is anomalous
by looking at its support and testing if it is included in the lines selected during training.
If it is not, then it is labeled as an outlier.
In our experiments we found that the set inclusion based anomaly detection technique can sometimes be a bit too conservative
which encourages false positives.

One Class Support Vector Machines (OC-SVM)~\citep{ocsvm}
is the unsupervised variant of the standard SVM problem
where the inliers are separated from the outliers
by the hyperplane separating the origin from most data points.
OC-SVM is a standard anomaly detection algorithm
that has been extended and integrated with multiple other methods
including autoencoder neural networks~\citep{Ruff18_DeepSVDD,Nguyen19_DeepOCSVM}.

\paragraph{Contributions.} The main contributions in our paper are listed as:
\begin{enumerate}
\item[$(i)$] 
We propose a new anomaly detection model
that adapts and extends the DL formulation
to include the OC-SVM objective,
where OC-SVM takes as input the DL-based sparse representations. This will produce a strong coupling between sparse representations and vector machines
which will permit us to 
replace testing for outliers via set inclusion
by a simple OC-SVM call.
\item[$(ii)$] Besides introducing a new anomaly detection model, we also provide theoretical results and algorithms
for our two problem formulations.
\item[$(iii)$] We further extend both objectives to more general setting that allows the use of kernel functions.
\item[$(iv)$] We provide extensive numerical tests that show the empirical quality of our schemes over shallow widely known learning algorithms
\item[$(v)$] Finally, we complete the comparative numerical tests with deep autoencoders 
showing empirically that our methods provide similar or better results
\end{enumerate}

\paragraph{Outline}
In Section~\ref{sec:preliminaries} we discuss and connect our proposal to the state of the art
and
continue to describe the anomaly detection methodology and general approach in Section~\ref{sec:methodology}.
Section~\ref{sec:linear} presents our main theorem
describing the iterations of the mixed DL-OCSVM objective.
In Section~\ref{sec:kernel} we introduce non-linearity via kernel methods
and adapt the algorithms and the theoretical framework to the new formulations.
To validate our theory,
in Section~\ref{sec:results}
we present extensive numerical simulations
against the state of the art
comprising of tests on real-data from multiple domains.
We conclude in Section~\ref{sec:conclusions}.



\paragraph{Notations}
Let $a_i$ be the $i$-th column, $a^j$ the $j$-th line and $a_{ij}$ an element of matrix $A$.
Where no indices are present,
and unless otherwise specified,
$a$ represents a column vector
and $a^T$ a line vector.
As a general rule we use lowercase for vectors and uppercase for matrices.
We denote with $\norm{A}_F=\sqrt{\sum_{i,j} a_{ij}^2}$ the Frobenius norm
and with $\norm{a}$ the $\ell_2$ vector norm.
The pseudo-norm $\norm{\cdot}_{0}$ counts the number of nonzero elements of a vector.
We denote the restriction to column normalized matrices as $\N_{m,n} = \{D \in \rset^{m \times n}\!:\! \norm{d_j}  =1, \forall j \in [n] \}$,
where $[n] = \{1, \cdots, n\}$ for $n \ge 1$.

\section{Preliminaries}
\label{sec:preliminaries}

We are interested in the dictionary learning problem where
the training data set has full-rank and
the dictionary matrices have normed columns (also called atoms).
The standard dictionary learning problem is
\be
D^\star,X^\star = \arg\min_{X, D \in \N_{m,n}} \norm{Y-DX}_F^2
\label{dictlearn}
\ee
where $d_j$ is the $j$-th atom of dictionary $D$
and $X$ is the sparse representation matrix with support $\Omega$
such that $X_\Omega = \{x_{ij} \mid x_{ij}\neq 0\!:\! \forall i \in [n], j \in [N] \}$.
Our current work assumes that the support $\Omega$ is either given
or obtained through a good sparse representation algorithm such as Orthogonal Matching Pursuit (OMP)~\citep{PRK93omp}.

Problem \eqref{dictlearn} is hard
and current approaches are usually based on the K-SVD algorithm~\citep{AEB06}
which is a block-coordinate descent algorithm 
iterating over pairs of atoms $d_i$ and their corresponding rows $x^i \in X$
solving
\be
(d_i^+, (x^i)^+) = \underset{d_i, x^i: \norm{d_i}=1}{\arg\min} \frac12\norm{d_ix^i - R}_F^2
\label{ksvd}
\ee
where $R = \left[ Y - \sum_{j \neq i} d_j x^j \right]_{\Omega^i}$ is the error $E = Y - DX$ with the $i$-th pair removed and restricted only to the signals in $Y$ that use atom $d_i$ in their representation.
This implies that $x^i$ in \eqref{ksvd}
contains only the non-zero entries of row $i$ from $X$
which corresponds to the support from row $i$ of $\Omega$.
Indeed,
this changes nothing in our computations
as including the zero entries would only generate zero columns in $d_ix^i$
which can not reduce the residual $R$.
Thus,
K-SVD iterations update the atoms and their associated coefficients in $X$ without altering $\Omega$.
We will use this fact and the restricted form throughout the paper.

\subsection{Uniform support for unsupervised anomaly detection}
\label{sec:icassp}

When obtaining $\Omega$,
sparse representation algorithms in general,
and OMP in particular,
focus on column based sparsity
(e.g. each $x_i \in X$ is at most $s$-sparse).
Of important note here is the fact that OMP provides unstructured sparsity: given $x_i$ and $x_j$ from $X$,
it is with high probability that the two lie in different subspaces
and thus their support differs;
on the other hand a row from $X$ has any number of zeros and non-zeros and differs completely from the next.
In the rare case where a row of $X$ is zero,
this implies that the corresponding atom from $D$ is unused,
thus some K-SVD variants perform atom replacement after each iteration $\eqref{ksvd}$
if $(x^i)^+ = 0$
in the hope that the new atom will be useful at the next iteration
when recomputing $\Omega$.

This is unlike structured sparisty approaches~\citep{Jolliffe16_ReviewPCA}
such as
Principal Component Analysis~(PCA)~\citep{Pearson1901_PCA}
or
Joint Sparse Representation~(JSR)~\citep{LiZha:15} 
where each $x \in X$ lies on the same subspace.
Here the rows of $X$ are either zero or filled with non-zeros.

In our pursuit for unsupervised anomaly detection,
we introduced a mix of the two approaches in~\citep{IRP22_ClipDL},
that attacks a regularized formulation of the \eqref{dictlearn} objective
\be
D^\star,X^\star = \arg\min_{D,X} \norm{Y-DX}_F^2 + \beta\sum_i\phi(\norm{x^i})
\label{supp_dictlearn}
\ee
where $\phi$ is a sparse regularizer on the rows in $X$.
Please note the following particular forms of $\phi$:
when $\phi(z) = z$ we recover the widely known sparse penalty $\norm{X}_{2,1}$
also known as the $\ell_{2,1}$-norm;
similarly $\phi(z) = \norm{z}_0$ leads to the $\ell_{2,0}$-norm
and
$\phi(z) = \ell_{\epsilon}(z):= \min\{|z|,\epsilon\}$
is the truncated $\ell_{2,\varepsilon}$-norm.
The resulting algorithm
starts from an unstructured support obtained through OMP
and then performs KSVD-like iterations that follow
\be
(d_i^+, (x^i)^+) = \arg\min_{d_i: \norm{d_i}=1, x^i} \frac{1}{2}\norm{d_{i}x^i -R^k}_F^2 + \beta\phi(\norm{x^i}_2)
\label{ksvd_supp_iter}
\ee
When $\phi(z) = z$,
if $\sigma_1 \ge \beta$
the iteration has the closed form
$(d_{i},x^i)= \left(u_1, (\sigma_1 - \beta) v^1\right)$
otherwise $(d_{i},x^i) = (d_{i}, 0)$
where $\sigma_1 u_1 v^1$ is the first SVD term of $R$.
When the condition is not met,
the entire row is zeroed in PCA fashion
resulting in a uniform support.
Please note that at the end
the remaining non-zero rows,
denoted with $\I$,
still have an unstructured sparsity pattern
(i.e. $X_\I$ resembles an OMP generated sparsity pattern).
This process is called uniform sparse representation and,
as described in the Introduction,
given a test point $\Tilde{y}$ whose sparse representation $\Tilde{x}$
is obtained through the learned dictionary $D$ via OMP,
we perform anomaly detection through testing if the support of $\Tilde{x}$
is included in $\I$.

In this paper we integrate the OC-SVM objective with the uniform support approach in \eqref{supp_dictlearn}
in order to obtain an unsupervised self-contained anomaly detection method
that improves and provides a more theoretically sound anomaly detection technique.
Also, due to the minor numerical differences between sparsity regularizers $\phi$ from \citep{IRP22_ClipDL},
our theorems will only follow the $\ell_{2,1}$ penalty where $\phi(z)=z$;
the non-convex $\ell_{2,0}$ and $\ell_{2,\varepsilon}$
regularizers are thus omitted but can be easily adapted in our proposed framework from the following sections.
\subsection{Supervised pair learning with vector machines}
\label{sec:prelim_DPL}

To our knowledge the only existing work that attempted to integrate dictionary learning with support vector machines is presented in~\citep{DL_NN} where,
unlike us,
the authors tackle the general supervised multi-class machine learning setting.
The dictionary learning part is based on dictionary pair learning~(DPL) for classification~\citep{DL_DPL}
\be
D^\star, P^\star = \underset{D,P}{\arg\min} \sum_{k=1}^{K} \norm{Y_k - D_k P_k Y_k}_F^2
                 + \gamma \norm{P_k {\bar Y}_k}_F^2
\label{DPL}
\ee
where we are given $K$ classes with a labeled dataset $Y$
such that the data $Y_k$ belong to class $k$ and the ${\bar Y}_k$ data do not.
For each class we learn two dictionaries.
$D_k\in\rset^{m\times n_k}$ is the synthesis dictionary from \eqref{dictlearn}
that brings column sparsity in class $k$ representations $X_k\in\rset^{n_k \times N}$.
In the DPL formulation,
the representations $X$ are no longer obtained through a separate sparse representation algorithm,
but instead through another dictionary learning problem such that $X_k = P_kY_k$.
The $P_k\in\rset^{n_k \times m}$ dictionary is the analysis dictionary
with most of its rows orthogonal to any given signal from $Y_k$.

In \eqref{DPL},
the resulting representations $X$ follow a diagonal block form,
$\norm{Y_k - D_k P_k Y_k}_F^2 \le \norm{Y_i - D_k P_k Y_i}_F^2 \; \forall i \neq k$,
which can be viewed as clustering through regularization.
Given a test point $z$,
classification is performed through
$c = \underset{k}{\arg\min}\norm{z - D_k P_k z}^2$.

The authors in \citep{DL_NN}
extend the DPL objective to include 
support vector machines with dictionary learning for general supervised classification tasks
\begin{align}
D^\star, P^\star, X^\star, \omega^\star, \rho^\star &= 
\underset{D,P,X,\omega,\rho}{\arg\min} \sum_{k=1}^{K} \norm{Y_k - D_k X_k}_F^2
  + \gamma_1 \norm{P_k Y_k - X_k}_F^2\nonumber \\
  &+ \gamma_2 \norm{P_k {\bar Y}_k}_F^2
  + \gamma_3 g(X_k, h^k, \omega_k, \rho_k)
  + \gamma_4 \norm{P_k}_F^2.
\label{DL-NN}
\end{align}
Here we are presented with a minor adaptation of the DPL problem
that is split among the first three terms
such that
the first term is the standard DL problem~\eqref{dictlearn}
providing $X_k$
to the second term that learns the analysis dictionary
based on the $X_k=P_kY_k$ DPL factorization.
In the fourth term $g$ is the standard SVM function
applied on the representations $X$ where $h$ are the labels
and $(\omega,\rho)$ the separating hyperplane.
The last term is a penalty on the magnitude of the coefficients from $P$.
To solve \eqref{DL-NN}
the authors employ an alternating minimization (AM) scheme on each of the variables.
Given a test point $z$,
classification is performed through
$c = \underset{k}{\arg\min}~\alpha \norm{z - D_k P_k z}^2
  - \left( \omega_k^T P_kz + \rho_k \right)$
where $\alpha>0$ is an extra hyper-parameter meant to balance the two terms.

\subsection{Related work}

Although mostly under the supervised setting and without a mix between DL and SVM,
there exists a lot of recent work on this topic which shows that there is an active community with a keen interest in the area.

We start off with existing work employing supervised DL with SVM.
\citep{Song23_RelaxedSVM_DL} revisit \eqref{DL-NN}
and instead propose an ADMM-based approach for attaining the objective,
\citep{Song18_CW-DDL} impose smoothness and label-based profiling on the sparse coefficients by augumenting \eqref{DL-NN}.
\citep{Yang21_SVM-DDL} is very similar to the work proposed in \citep{DL_NN}
but in the standard DL setting \eqref{dictlearn} with Fisher-based inter and intra-class discrimination~\citep{YZFZ11_FisherDL}.
\citep{Cai14_SVGDL} follows the same approach but without Fisher discrimination.
Although the SVM model is not taken into consideration when updating the representations,
\citep{Song19_MDDL} 
proposes a convolutional Fisher DL supervised discrimination scheme
that uses an SVM classifier at the end on the resulting representations.

Within the DPL family we first mention supervised classification works such as
\citep{Deng23_DPL-AD} that performs anomaly detection using DPL for process monitoring,
\citep{Zhu23_DeepPairDL} employing multi-class DPL in the latent space of an autoencoder architecture,
\citep{Chen23_DPL-SCSR} imposing structure on the multi-class DPL based algorithm.
Lastly,
tackling only the DPL without classification
but with a direct impact on the above
we mention
\citep{Wang21_ERDDPL} that use FISTA for sparse representations and alternating schemes for learning the dictionaries,
and also \citep{Du21_SDADL} which extend the supervised discriminative setting to the DPL formulation.
Other non-linear convolutional or autoencoder DL-based supervised classifictation
techniques can be found in \citep{Abdi19_DLE,Dong23_AOLP}.

In the anomaly detection realm we found little existing work.
\citep{Sadeghi14_OutlierSupervisedDL} provides supervised learning where the training set contains labeled outliers and the standard DL problem is penalized with $\norm{O}_{2,1}$ where $O$ is the set of outliers.
\citep{Forero17_GraphOutlierSupervisedDL} extends the former with an extra graph smoothing constraint.
Besides the uniform support approach \citep{IRP22_ClipDL},
for unuspervised anomaly detection with DL we found only
\citep{Whitaker15_MinMaxErrDL}
where the authors change the matrix norm in the standard DL problem in order to minimize the maximal error,
a simple method that resembles the weak trimming approach.

It is clear that
the paper topic is of interest and
that our proposed unsupervised anomaly detection method based on the mix between DL and OC-SVM has little coverage in existing literature.
In the following sections we will analyze both the standard DL and DPL formulations
and their non-linear adaptation through kernel methods.

\section{Anomaly Detection Methodology and Algorithms}
\label{sec:methodology}

\begin{algorithm}[t]
\DontPrintSemicolon
\SetKwComment{Comment}{}{}
\KwData{train set $Y$,
test set $\Tilde{Y}$,
dictionary $D^0$,
sparsity $s$,
iterations K
}
\KwResult{final dictionary $D$ and OC-SVM model $(\omega, \rho, \xi, \lambda)$}
\BlankLine
\emph{Training Procedure }\\
Representation: apply OMP to obtain $X^0$ and its support $\Omega$\\
Support vectors: model $(\omega^0, \rho^{0}, \xi^{0},\lambda^{0})$ based on $X^0$ \\
\For{$k \in \{1,\dots,K \}$}{
DL: inner routine that learns $D^{k}$ and $X^{k}$ with fixed $\omega^{k-1}, \lambda^{k-1}$ \\
Support vectors: model $(\omega^{k}, \rho^{k}, \xi^{k}, \lambda^{k})$ based on fixed $X^{k}$ \\
}
Inference: build anomalies set $\A$ based on OC-SVM at iteration $K$

\BlankLine
\emph{Anomaly Detection}\\
Representation: inner routine that obtains $\Tilde{X}$ from $\Tilde{Y}$\\ 
Inference: build anomalies set $\Tilde{\A}$ using OC-SVM trained model on $\Tilde{X}$\\
\caption{Anomaly Detection Scheme}
\label{alg:general_scheme}
\end{algorithm}

Our general model minimizes the combined loss function
\begin{align}
\mathcal{L}(Y, D, X, \omega, \rho, \xi)  &= F(Y, D, X) + G(X,\omega,\rho,\xi),
\label{general_model}
\end{align}
where $F$ represents the dictionary learning objective
and, respectively,
$G$ the one-class support vector machine (OC-SVM) objective.
We can now define the base model of our current study.
Let $F$ follow the anomaly detection with uniform sparse representation from \eqref{supp_dictlearn}
and let $G$ be the standard OC-SVM problem
\begin{align}
F(Y, D, X)  &= \frac12 \norm{Y - DX}_F^2 
  + \beta \sum_{i=1}^{n} \phi\left(\norm{x^i}_2\right)\label{icassp} \\
G(X, \omega, \rho, \xi) &=\max\limits_{\lambda \ge 0} \; \norm{\omega}^2 - \rho + \frac{1}{CN} \sum_{i=1}^{N} \xi_i 
  - \sum_{i=1}^{N} \lambda_i (\omega^T x_i - \rho + \xi_i)\label{ocsvm}.
\end{align}
Here we use the Lagrangian form of the OC-SVM problem
where
the representations $X$ produced through the dictionary learning process are used as support vectors,
$\omega$ and $\rho$ define the separating hyperplane,
$\xi$ are slack variables
and
$\lambda$ the Lagrange multipliers.
In the DPL formulation the objective will be modified accordingly
as $X = PY$ implies that the variable $X$ is replaced by $P$ and the loss function becomes $\mathcal{L}(Y, D, P, \omega, \rho, \xi) = F(Y, D, P) + G(PY,\omega,\rho,\xi)$.

We propose an alternating minimization~(AM) approach to solve the objective from \eqref{general_model}.
The steps necessary for our general anomaly detection scheme are organized in Algorithm~\ref{alg:general_scheme}.
Steps 2--6 describe the training procedure.
Step 2 computes the initial representations using the initial dictionary and column sparsity $s$.
The resulting support $\Omega$ will be fixed during the rest of the procedure unless an entire row from $X$ is zeroed.
Step 3 trains an initial OC-SVM model based on the initial representations.
The AM iterations from steps 5 and 6,
also called outer iterations,
are performed $K$-times by step 4.
Step 5 performs the inner-iterations of the DL process mixed with the OC-SVM objective
and at the end, in step 6,
the OC-SVM model is updated with the new representations
from the previous step.

There are two approaches to unsupervised anomaly detection.
The first approach
trains a model on the entire dataset (steps 1--6) and then provides the found anomalies in set $\A$ (step 7).
The second approach (steps 8--10)
trains a model on a data subset
and then checks the rest of the data for anomalies by
performing sparse representation with the trained dictionary (step 9)
and inferring on the resulting representation with the trained OC-SVM model (step 10).
We will investigate both approaches,
while the former may provide better results because it sees all the data in training,
the later can provide faster training times due to the reduced data size and can be less prone to over-fitting errors.
Given a test point $z$ from the dataset,
OC-SVM inference (at step 7 or 10) computes
$h = \sign \left(\omega^T z - \rho \right)$.
If $h\le0$ then we found an anomaly and we add it to set $\A$.

The following sections will analyze the algorithm properties together with the DPL and kernel extensions.

\section{Uniform representations through regularization and OC-SVM}
\label{sec:linear}

In this section we will study the adaptation of standard DL and the DPL formulations for uniform representations mixed with OC-SVM and their effect on the general loss $\mathcal{L}$.

\subsection{Standard DL with \texorpdfstring{$\ell_{2,1}$}{L21} regularization}
\label{sec:icassp_ocsvm}

Given the mixed objective \eqref{general_model} with $\phi(z) = z$ in
\eqref{icassp} we obtain
\begin{align}
\mathcal{L}(Y, D, X, \omega, \rho, \xi)  &= \frac12\norm{Y - DX}_F^2  + \beta  \sum_{i=1}^n \norm{x^i}_2 + G(X,\omega,\rho,\xi),
\label{icassp_ocsvm}
\end{align}
where $x^i$ represents a line from  $X$. Let us consider the original K-SVD iteration where each atom $d_i$ and line $x^i$, are updated in a cyclic manner, at once. 
Our mixed objective becomes:
\begin{align}
(d^+_i,(x^i)^+) = \underset{d_i: \norm{d_i}=1,x^i}{\arg\min} \frac12\norm{d_ix^i - R}_F^2  + \beta\norm{x^i}_2 - \omega_i x^i\lambda ,
\label{K-SVD_iteration}
\end{align}
where $R = Y - \sum_{j \neq i} d_j x^j$ is the approximation residual without the contribution of $d_i$ and $x^i$,
and, from the OC-SVM objective,
$\omega_i \in \omega$ is the element corresponding to $d_i$
and $\lambda$ the Lagrange multipliers.
The success of K-SVD updates for sparse representation problems \citep{AEB06} relied upon the explicit form of the new iterates pair $(d^+,x^+)$ and the fast convergence to a good stationary point. Although the K-SVD extension developed in \eqref{K-SVD_iteration} does not share the former benefit that allows its easy computation, we further formulate the new iterates as the solution of tractable quadratic minimization problems.

\begin{theorem}\label{prop:icassp_ocsvm_l1}
Let $ \phi(z) = z $  and $\nu$ the element of $\omega$ corresponding to the current atom $d_i$. 
Then the closed form solution of the K-SVD iteration for \eqref{K-SVD_iteration} is given by: let $d^* = \arg\min \limits_{\norm{v} = 1} \;  -\frac{1}{2}\norm{R^Tv + \nu\lambda}^2_2$

\noindent $(i)$ If $\norm{R^Td^* + \nu\lambda} \ge \beta$ then
\begin{align}\label{opt_icassp_ocsvm_l1}
    d_i^+ & = d^* \\
	(x^i)^+ & = \left(1 - \frac{\beta}{\norm{R^T d_i^+ + \nu\lambda}} \right)(R^T d_i^+ + \nu\lambda)^T, \label{opt_icassp_ocsvm_l1_x}
\end{align}

\noindent $(ii)$ Otherwise, when $\norm{R^Td^* + \nu\lambda} < \beta$,  $( d_i^+,(x^i)^+) = (d_i, 0)$.
\end{theorem}

\begin{proof}

For simplicity we drop the coordinate indices. The modified K-SVD iteration \eqref{K-SVD_iteration} becomes:
\begin{align*}
 \min\limits_{\norm{d} = 1}  \min\limits_{x} &  \; \frac{1}{2}\norm{x}^2 - xR^Td + \frac{1}{2} \norm{R}^2_F + \beta\norm{x} - \nu x\lambda,
\end{align*}
Denote $x^*(d)$ the solution in $x$ for a fixed $d$.
The first order optimality of the inner problem is given by: $0 \in x^*(d) - (R^Td + \nu\lambda) + \beta \partial \norm{\cdot}(x^*(d))$.
If $\norm{R^Td + \nu\lambda} \le \beta$ then $x^*(d) = 0$, otherwise the above relation reduces to:
\begin{align}\label{opt_cond_case2}
	\left(1 + \frac{\beta}{\norm{x^*(d)}} \right) x^*(d)^T = R^Td + \nu\lambda.
\end{align}
Then, a first consequence of \eqref{opt_cond_case2} is $\norm{x^*(d)} = \norm{R^Td + \nu\lambda} - \beta. $
By using this identity again in \eqref{opt_cond_case2} yields as a second consequence
$\left(1 + \frac{\beta}{\norm{R^Td + \nu\lambda} - \beta} \right) x^*(d)^T = R^Td + \nu\lambda$
which produces in this case the explicit form of the solution $x^*(d)$
\begin{align}\label{x_optimal}
	x^*(d) = \left(1 - \frac{\beta}{\norm{R^Td + \nu\lambda}} \right)(R^Td + \nu\lambda)^T.
\end{align}
Replacing both null and explicit form in the original minimization problem we get:
\begin{align*}
	&  \; \frac{1}{2}\norm{x^*(d)}^2 - x^*(d)R^Td + \beta\norm{x^*(d)} - \nu x^*(d)\lambda \\
	= &  \begin{cases}
	\; -\frac{1}{2}\left( \norm{R^Td + \nu\lambda} -  {\beta} \right)^2 & \text{if} \; \norm{R^Td + \nu\lambda} > \beta \\
	0 & \text{otherwise}.
	\end{cases} = -\frac{1}{2} \left(\norm{R^Td + \nu\lambda} -  {\beta}\right)_+^2.
\end{align*}
where $(\cdot)_+:= \max\{0,\cdot\}$. Since $(\cdot - \beta)_+^2$ is nondecreasing, then the K-SVD subproblem reduces to computing
$d^* = \arg\max\limits_{\norm{d} = 1} \; \frac{1}{2}\norm{R^Td + \nu\lambda}^2_2 $,
followed by computing $x^*(d^*)$ through \eqref{x_optimal}. 

\end{proof}

\begin{remark}
Computing $(d^i)^+$ with \eqref{opt_icassp_ocsvm_l1}  reduces to maximizing a convex quadratic cost over the sphere, which results in a nonconvex constrained optimization problem commonly named as \textit{trust region subproblem} \citep{BenTeb:96}. Despite its nonconvex nature, there are various references stating that the global solution is computable through simple iterative schemes.
Given the singular value decomposition $R = U\Sigma V^T$, the subproblem from \eqref{opt_icassp_ocsvm_l1} 
\begin{align*}
\min \limits_{\norm{d} = 1} \; -\frac{1}{2}\norm{R^Td + \nu\lambda}^2_2 = -\frac{1}{2}d^TRR^Td - \nu d^TR\lambda + c ,
\end{align*}
can be reformulated through a change of variables into $\bar{d}:=U^Td$ to get:
\begin{align*}
\min \limits_{\norm{\bar{d}} = 1} \;  -\frac{1}{2}\bar{d}^T\Sigma^2\bar{d} -  \bar{d}^T q + c,
\end{align*}
where $q:=\nu \Sigma V^T\lambda$.
As stated in \citep[Corrolary 8]{BenTeb:96}, by solving the bi-dual convex problem:
\begin{align}\label{convex_bidual}
\min\limits_{\sum_i y_i = 1} \;  -\frac{1}{2} \sum_i \sigma^2_i y_i -  |q_i| \sqrt{y_i}.
\end{align}
we recover the optimal $\bar{d}^*_i = \sign{(q_i)}\sqrt{y_i^*}$ and further more $d^*= U \bar{d}^*$. The final problem \eqref{convex_bidual} is convex and it has a single linear equality constraint.  
\end{remark}

The exact computation of $d^+$ based on the  reasoning from above is computationally dominated by the SVD procedure at cost $\mathcal{O}(m^2 N)$. 
The processed program \eqref{convex_bidual} has been particularly solved in \citep{BenTeb:96} by a simple interior-point scheme in $\mathcal{O}(m \ln(2m/\epsilon))$, where $\epsilon$ is the accuracy of the solution. 
For this reason, we used in practice a particular Power Method scheme that approximates the solution \eqref{opt_icassp_ocsvm_l1} and avoids the SVD step.


\begin{algorithm}[t]
\DontPrintSemicolon
\SetKwComment{Comment}{}{}
\KwData{train set $Y \in \rset^{m \times N}$,
$D^{k-1} \in \rset^{m \times n}$,
inner iteration k
}
\KwResult{dictionary $D^k$ and representations $X^k$}
\BlankLine

Error: $E^k = Y - D^{k-1} X^{k-1}$ \\
\For{$i \in \{1,\dots,n\}$}{
    Atom error: $R^k = E^k + d_{i}^{k-1} x^{i,k-1}$ \\
    Update: new $(d_{i}^{k}, x^{i,k})$ according to Th.
    \ref{prop:icassp_ocsvm_l1} \\
    New error: $E^{k} = R^k - d_{i}^{k} x^{i,k}$ \\
}
\caption{DL-OCSVM Uniform Representation Learning}
\label{alg:ksvd-ocsvm-training}
\end{algorithm}

We are now ready to describe in Algorithm~\ref{alg:ksvd-ocsvm-training} the inner-iterations of the DL process (see step 5 in Algorithm~\ref{alg:general_scheme}).
Given the initial dictionary and representations from the last outer iteration,
step 1 initializes the overall representation error $E$.
Step 2 starts the inner iterations walking through all the dictionary atoms and their associated rows in $X$.
In step 3 we compute the residual of the current pair
and proceed to update the atom and its representations
at step 4 following Theorem~\ref{prop:icassp_ocsvm_l1}.
Finally,
in step 5,
we update the error with the new atom and representations.

\begin{proposition}\label{prop:fixed_point}
Let $\{D^k,X^k, \omega^k, \rho^k, \xi^k\}_{k \ge 0}$ be the sequence
generated by the Algorithm~\ref{alg:general_scheme}
using the inner iterations of Algorithm~\ref{alg:ksvd-ocsvm-training}.
Then the following decrease holds
\be
\mathcal{L}(Y, D^{k+1}, X^{k+1}, \omega^{k+1}, \rho^{k+1}, \xi^{k+1}) 
 \le \mathcal{L}(Y, D^{k}, X^{k}, \omega^{k}, \rho^{k}, \xi^{k}) \quad \forall k \ge 0.\nonumber
\ee
\end{proposition}
\begin{proof}
\begin{align*}
&\mathcal{L}(Y, D^{k}, X^{k}, \omega^{k}, \rho^{k}, \xi^{k})
 = F(Y, D^k, X^k) + G(X^k,\omega^{k},\rho^{k},\xi^{k}) \\
 &= \frac12 \norm{R^k - d_i^kx^{i,k}}_F^2 
 + \beta \norm{x^{i,k}} + \beta \sum_{j\neq i} \norm{x^{j,k}} \\
 &\quad\quad + \norm{\omega^{k}}^2 - \rho^{k} + \frac{1}{CN} \sum_i \xi_i^{k} 
  -\omega_i^{k}x^{i,k}\lambda 
  - \sum_{j\neq i} \omega_j^{k} x^{j,k} \lambda  
  +\sum_{j} \lambda_j \left[- \rho^{k} + \xi_j^{k}\right] \\
&\overset{Th.\ref{prop:icassp_ocsvm_l1}}{\ge} \frac12 \norm{R^k - d_i^{k+1}x^{i,k+1}}_F^2 
 + \beta \norm{x^{i,k+1}} -\omega_i^{k}x^{i,k+1}\lambda 
 + \beta \sum_{j\neq i} \norm{x^{j,k}} \\
 &\quad\quad
 + \norm{\omega^{k}}^2 - \rho^{k} + \frac{1}{CN} \sum_i \xi_i^{k} 
 - \sum_{j\neq i} \omega_j^{k} x^{j,k} \lambda  
  +\sum_{j} \lambda_j \left[- \rho^{k} + \xi_j^{k}\right] \\
&\overset{Th.\ref{prop:icassp_ocsvm_l1}}{\ge} \frac12 \norm{R^k - d_i^{k+1}x^{i,k+1} - d_{i+1}^{k+1}x^{i+1,k+1}}_F^2 
 + \beta\left(\norm{x^{i,k+1}} + \norm{x^{i+1,k+1}}\right) \\
 &\quad\quad  - \left(\omega_{i}^{k}x^{i,k+1} + \omega_{i+1}^{k}x^{i+1,k+1}\right)\lambda
 + \beta \sum_{j\neq i,i+1} \norm{x^{j,k}} \\
 &\quad\quad 
 + \norm{\omega^{k}}^2 - \rho^{k} + \frac{1}{CN} \sum_i \xi_i^{k} 
 - \sum_{j\neq i,i+1} \omega_j^{k} x^{j,k} \lambda  
  +\sum_{j} \lambda_j \left[- \rho^{k} + \xi_j^{k}\right] \\
&\ge F(Y, D^{k+1}, X^{k+1})
 + \norm{\omega^{k}}^2 - \rho^{k} + \frac{1}{CN} \sum_i \xi_i^{k} 
 - \sum_{j} \omega_j^{k} x^{j,k+1} \lambda  \\
&\quad\quad +\sum_{j} \lambda_j \left[- \rho^{k} + \xi_j^{k}\right]
= F(Y, D^{k+1}, X^{k+1}) + G(X^{k+1},\omega^{k},\rho^{k},\xi^{k}) \\
&\overset{\text{OC-SVM}}{\ge} F(Y, D^{k+1}, X^{k+1}) + G(X^{k+1},\omega^{k+1},\rho^{k+1},\xi^{k+1}) \\
&= \mathcal{L}(Y, D^{k+1}, X^{k+1}, \omega^{k+1}, \rho^{k+1}, \xi^{k+1})
\end{align*}
where 
$G(X^{k+1},\omega^{k},\rho^{k},\xi^{k})$
  implied just an inference on $X^{k+1}$ with the existing model $(\omega^k, \rho^k, \lambda^k)$
and
$G(X^{k+1},\omega^{k+1},\rho^{k+1},\xi^{k+1})$
  implies an update of the support vectors with the new signals $X^{k+1}$ and an inference afterwards
thus making the last inequality true.
\end{proof}
Proposition~\ref{prop:fixed_point} and its proof
can be easily adapted to the algorithms
regarding the DPL formulation and the kernel variants from the following sections
and so we will not revisit the topic.

\begin{algorithm}[t]
\DontPrintSemicolon
\SetKwComment{Comment}{}{}
\KwData{test set $\Tilde{Y} \in \rset^{m \times \Tilde{N}}$,
dictionary $D \in \rset^{m \times n}$, sparsity $s$, \\
 and OC-SVM model $(\omega, \rho, \lambda)$
}
\KwResult{anomalies $\Tilde{\A}$}
\BlankLine
Representation: $\Tilde{X} = \text{OMP}(\Tilde{Y}, D, s)$ \\
Error: $E = \Tilde{Y} - D \Tilde{X}$ \\
\For{$i \in \{1, \dots, n\}$}{
    Atom error: $R = E + d_i \Tilde{x}^{i}$ \\
    Trimming: zero $\Tilde{x}^{i}$ if condition from \eqref{ksvd_supp_iter} holds\\
}
\lIf{$\sign \left( \omega^T \Tilde{x}_i - \rho \right) \le 0$}{$\Tilde{\A} = \Tilde{\A} \cup \{i\} \ \;\ \forall i \in \Tilde{N}$}

\caption{DL-OCSVM Anomaly Detection}
\label{alg:ksvd-ocsvm-testing}
\end{algorithm}

We now focus on Anomaly Detection with a trained model resulted from Algorithm~\ref{alg:general_scheme} with the inner iterations of Algorithm~\ref{alg:ksvd-ocsvm-training}.
First note that after training is completed the OC-SVM term
$\omega_i x^i \lambda$ from \eqref{K-SVD_iteration} 
vanishes as new test data does not affect nor depend on the training data and the associated $\lambda$ parameters.
Thus \eqref{K-SVD_iteration} resumes to problem \eqref{ksvd_supp_iter} described in Section~\ref{sec:icassp}.
Taking this in consideration during testing,
for uniform representation
we check the condition $\sigma_1 < \beta$
and zero the row $x$ when it holds
as described below \eqref{ksvd_supp_iter}.

For clarity we include in Algorithm~\ref{alg:ksvd-ocsvm-testing}
all the necessary steps for performing anomaly detection with the trained model on a different set $\Tilde{Y}$. Steps 1 and 2 produce the representations with the provided dictionary and compute the representation error similar to the initialization during training.
Step 3 starts the inner iterations
walking the dictionary atoms without updating them
but instead computing the residual (step 4)
and zeroing the associated representations row
if the condition from \eqref{ksvd_supp_iter} holds.
Finally,
in step 6
we apply OC-SVM inference on
the resulting sparse representations $\Tilde{X}$.

\subsection{Adapting the DPL formulation}
\label{sec:icassp_DPL}

As we have seen in Section~\ref{sec:preliminaries},
the DPL formulation is often used in the literature
mostly due to its efficient sparse representation computation:
it is much faster to compute the matrix multiplication $PY$
than to apply a greedy algorithm such as OMP.

Thus let us investigate next the DPL formulation by modifying $F$ in \eqref{icassp} to include the $DPY$ factorization
\begin{align}
F(Y, D, P)  &= \frac12 \norm{Y - DPY}_F^2 
  + \beta \sum_{i=1}^n \phi\left(\norm{p^iY}_2\right)
  + \sum_{i=1}^n\alpha_i \norm{p^iY}_1
\label{icassp_DPL}
\end{align}
where $D$ is the synthesis dictionary
and
$P$ is the analysis dictionary.
In this formulation
the sparse representations $X$ are obtained through the simple matrix multiplication $PY$.
We add the third penalization term $\sum_{i=1}^n\alpha_i \norm{p^iY}_1$ in order to induce zeros on the rows of matrix $X$. This is helpful because, unlike multi-class classification schemes such as \citep{DL_DPL} and \citep{DL_NN} which rely on group sparsity
using per-class matrices $D_k$ and $P_k$
as discussed in Section~\ref{sec:prelim_DPL},
in our approach we have a single $D$ and $P$ and we rely on the individual sparse representations for discrimination.

\begin{remark}
\textbf{Choosing $\alpha$ for the $\ell_1$ penalty.}
We have $Ns$ nonzeros in X, $s$ per each column, out of the total $Nm$ elements of $X$.
Let $\Omega$ be the initial support of $X$ generated by OMP
and $\Omega^i$ the support of row $x^i$ such that $\varsigma_i = \norm{x^i}_0 = \abs{\Omega^i}$.
We can set
$\alpha \in\rset^n$ as the unit vector whose
elements are built from the row-wise $\ell_0$-norm of $X$
as $\alpha_i = \varsigma_i / \sqrt{\sum_k{\varsigma_k^2}}$
which will then impose row-sparsity according to $\Omega$.
To cope with the other parameters in \eqref{icassp_DPL}
we can introduce a common parameter $\gamma$
and
rewrite the last term as $\gamma\sum_{i=1}^n\alpha_i \norm{p_iY}_1$.
\end{remark}

In K-SVD fashion,
we approach \eqref{icassp_DPL}
through the update of each synthesis atom $d$ in $D$
and its corresponding analysis atom $p$ from $P$
which produces the sparse representation line $x = pY$ from $X$.
For the convex $\ell_{2,1}$ regularization case 
where $\phi(z)=z$,
the modified K-SVD iteration becomes
\begin{align}
(d^+_i,(p^i)^+) = \underset{d_i: \norm{d_i}=1,p^i}{\arg\min} \frac12\norm{d_ip^iY - R}_F^2
+ \beta\norm{p^iY}_2 
+ \alpha\norm{p^iY}_1
- \omega_i p^iY\lambda ,
\label{dpl_iteration}
\end{align}
where $R = Y - \sum_{j \neq i} d_j p^jY$ is the approximation residual without the
$(d_i,p^i)$ pair.
The solution of \eqref{dpl_iteration} is obtained through the following theorem.

\begin{theorem}\label{prop:dpl_ocsvm_l1}
Assume $Y^\dagger = Y^T(YY^T)^{-1}$ exists and denote $H_Y = Y^\dagger Y $. 
Then the closed form solution of the K-SVD iteration for \eqref{icassp_DPL} is:
\begin{align}
    (d^+, \tau_1(d^+),\tau_2(d^+)) & = \min\limits_{\norm{d} = 1} \max\limits_{\overset{\norm{\tau_2}_{2}\le 1,}{\norm{\tau_1}_{\infty}\le 1}   }      - \frac{1}{2}\norm{R^Td + \nu \lambda  - \alpha \tau_1 - \beta \tau_2}^2_{H_Y} \label{opt_dpl_ocsvm_l1_d}\\
    p^+ & = (R^Td^+ + \nu \lambda - \beta \tau_2(d^+) - \alpha \tau_1(d^+))^T Y^\dagger,
    \label{opt_dpl_ocsvm_l1_x}
\end{align}
\end{theorem}

\begin{proof}
Let $p$ be an arbitrary line in matrix $P$. Therefore, substituting the representation line $x = pY$ in the K-SVD iteration we obtain: 
\begin{align}
 & \min\limits_{\norm{d} = 1}  \min\limits_{x = pY} 
   \; \frac{1}{2}\norm{x}^2 - x \left( R^Td + \nu \lambda\right) + \beta\norm{x}+ \alpha\norm{x}_1 \label{min_min_start}\\
 & = \; \min\limits_{\norm{d} = 1}  \min\limits_{p} 
   \; \frac{1}{2}\norm{pY}^2 - pY\left( R^Td + \nu \lambda\right) + \beta\norm{pY}+ \alpha\norm{pY}_1 \nonumber\\
& = \min\limits_{\norm{d} = 1}  \min\limits_{p}   \; \frac{1}{2}\norm{pY}^2 - pY\left( R^Td + \nu \lambda\right) +  \max\limits_{\norm{\tau_2} \le 1} \beta pY\tau_2 + \max\limits_{\norm{\tau_1}_{\infty}\le 1} \alpha pY\tau_1 \nonumber\\
& = \min\limits_{\norm{d} = 1}  \min\limits_{p} \max\limits_{\norm{\tau_2} \le 1, \norm{\tau}_{\infty}\le 1}  \; \frac{1}{2}\norm{pY}^2 - pY\left( R^Td + \nu \lambda\right) +  \beta pY\tau_2 + \alpha pY\tau_1 \nonumber\\
& = \min\limits_{\norm{d} = 1} \max\limits_{\norm{\tau_2} \le 1, \norm{\tau}_{\infty}\le 1} \min\limits_{p}   \;\frac{1}{2}\norm{pY}^2 - pY \left( R^Td + \nu \lambda -  \beta \tau_2 - \alpha \tau_1 \right) \label{min_in_p}\\
& = \min\limits_{\norm{d} = 1} \max\limits_{\norm{\tau_2} \le 1, \norm{\tau}_{\infty}\le 1}    \;  - \frac{1}{2}\norm{H_Y(R^Td + \nu \lambda  - \beta \tau_2 - \alpha \tau_1)}^2, \nonumber \\
& = \min\limits_{\norm{d} = 1} \max\limits_{\norm{\tau_2} \le 1, \norm{\tau}_{\infty}\le 1}    \;  - \frac{1}{2}\norm{R^Td + \nu \lambda  - \beta \tau_2 - \alpha \tau_1}^2_{H_Y}, \label{min_max_3var}
\end{align}
where in \eqref{min_in_p} we used the first order optimality conditions over variable $p$ and the form of the solution $p(d) = (R^Td + \nu \lambda - \beta \tau_2(d) - \alpha \tau_1(d))^T Y^\dagger$.
\end{proof}

\noindent Since the minimization step \eqref{dpl_iteration} is nonconvex and nonsmooth, we rely in our tests on the approximated updates \eqref{opt_dpl_ocsvm_l1_d}-\eqref{opt_dpl_ocsvm_l1_x}. For approximation, we used a projected gradient method with constant stepsize.   

Our algorithms require minor modifications to accommodate the DPL form.
The $(d,x)$ update from Algorithm~\ref{alg:ksvd-ocsvm-training} step 4 
is replaced with updating the $(d,p)$ pair as in Theorem~\ref{prop:dpl_ocsvm_l1}.
For anomaly detection with a trained model,
we no longer require steps 2 and 4 in Algorithm~\ref{alg:ksvd-ocsvm-testing}
and during trimming in step 5
we zero the lines in $\Tilde{X}$ where $\|p\Tilde{Y}\|$ is (almost) null.
In practice, following K-SVD, $R$ and $Y$ contain only the signals which use atom $d$ in their representation.
\section{Kernel formulation}
\label{sec:kernel}

Kernel functions $\varphi:\rset^m \to \rset^{\Tilde{m}}$ are used to
induce non-linearity in the feature space of $Y \in \rset^{m \times N}$
leading to the higher dimensional space $\varphi(Y) \in \rset^{\Tilde{m} \times N}$
where $\Tilde{m} \gg m$.
While this potentially improves our algorithms modeling capabilities,
it also comes with significant computational costs
that can be reduced through the \emph{kernel trick}
when Mercer functions are used.
If a kernel takes us to possibly infinite $\Tilde{m}$ dimensional feature spaces,
the kernel trick helps us reduce computations to $O(N)$ instead.
An improvement, but still far from our initial $m$ dimensionality.
\subsection{Kernel Dictionary Learning}
\label{sec:kernel-DL}

In order to employ the kernel trick,
the kernel dictionary learning problem
rewrites the dictionary as $D = \Tilde{D} + D_{\bot}$
with $\Tilde{D} \in \spanv(\varphi(Y))$
and $D_{\bot}$ its orthogonal complement
such that $\varphi(Y)^T D_{\bot} = 0$
and $\Tilde{D}^T D_\bot = 0$.
As shown in ~\citep[Ch.9]{dl_book}
inserting this in the standard dictionary learning problem \eqref{dictlearn}
leads to the equivalent problem
\be
\min_{A,X}\norm{\varphi(Y)(I - AX)}_F^2
\label{kernel_dictlearn}
\ee
where $D_\bot = 0$ and the dictionary $\Tilde{D}$ is rewritten as $D = \varphi(Y)A$
with $A\in\rset^{N\times n}$.
Indeed $\norm{Y-DX}_F^2 = \norm{\varphi(Y)(I-AX)}_F^2$,
whereas in the linear case (i.e. $\varphi(Y)=Y$)
it does not make sense to work with $AX$
as it involves more computation with no benefit,
for the kernel formulation it actually reduces algorithmic complexity
because when we unpack \eqref{kernel_dictlearn}
we form $K=\varphi(Y)^T\varphi(Y)$ with $K \in \rset^{N \times N}$
which enables the kernel trick.
We denote with $k_{ij} = \varphi(y_i)^T\varphi(y_j) = \kappa(y_i, y_j)$
the elements of $K$.

Let $z\in\rset^m$ be a given test point,
also
let $A\in\rset^{N\times n}$ be the trained kernel dictionary from \eqref{kernel_dictlearn}
and
let $\zeta = \kappa(y_i, z) = \{ \zeta_i \mid \zeta_i = \varphi(y_i)^T \varphi(z),\ \forall i \in [N] \}$.
Then the resulting representation error is
$e = \varphi(z) - \varphi(Y)Ax = \varphi(z) - \varphi(Y)A_\I x_\I$ with
$x_\I = \{ x_i \mid x_i \neq 0 \}$.
For OMP atom selection~\citep{Van13_kernelomp}
we never compute $e$
and instead we use the equivalent quantity
to evaluate the correlation
$D^Te = D^T[\varphi(z)-\varphi(Y)Ax] 
= A^T\varphi(Y)^T[\varphi(z)-\varphi(Y)Ax]
= A^T (\zeta - KAx)$.
For multiple test points $\Tilde{Y} \in \rset^{m\times\Tilde{N}}$,
where
$\Tilde{k_j} = \{ \Tilde{k}_{ij} \mid \Tilde{k}_{ij} = \varphi(y_i)^T \varphi(\Tilde{y}_j),\ \forall i \in [N] \}$
with $j \in [\Tilde{N}]$,
we arrive at $D^TE = A^T(\Tilde{K} - KAX) = A^T\Tilde{K} - A^TKAX$.
This can be solved by using $\text{OMP}(A^T\Tilde{K}, A^TKA, s)$
in step 1 of Algorithm~\ref{alg:ksvd-ocsvm-testing}
where $Y \leftarrow A^T\Tilde{K}$ and $D \leftarrow A^TKA$ .


In Step 2, the representation error also needs adjusting.
For the test set $\Tilde{Y}$,
the representation error becomes
$E = \varphi(\Tilde{Y}) - D\Tilde{X}  = \varphi(Y) \left(\Theta - A \Tilde{X} \right)$
where $\Theta$ is the expansion coefficients of $\varphi(\Tilde{Y})$
in the span of $\varphi(Y)$
such that
$\Theta = \varphi(Y)^\dagger \varphi(\Tilde{Y})
= \left( \varphi(Y)^T \varphi(Y) \right)^{-1} \varphi(Y)^T \varphi(\Tilde{Y})
= K^{-1} \Tilde{K}$.
In kernel DL algorithms,
the representation error 
is replaced with just $E = \Theta - A\Tilde{X}$,
working in the A space not in the feature space~\citep[Ch.9.4]{dl_book},
which leads to 
$E = K^{-1}\Tilde{K} - A\Tilde{X}$.
During training $\Tilde{K} \leftarrow K$ and we modify the above accordingly;
of special note is the training error $E=I-AX$ which circles back to the objective in \eqref{kernel_dictlearn}.

The standard Kernel K-SVD iteration follows
$\min_{a,x}\norm{\varphi(Y)(R - ax)}_F^2$ such that $\norm{K^{\frac12}a} = 1$.
The normalization constraint follows from
expanding the Frobenius norm in \eqref{kernel_dictlearn}
$(A^T\varphi(Y)^T)(\varphi(Y)A)=A^TKA=(A^TK^{\frac12})(K^{\frac12}A)$.
The kernel dictionary learning objective $F$ becomes
\be
F(Y, D, X) = \frac12 \norm{\varphi(Y)(I - AX)}_F^2 
  + \beta \sum_i \phi\left(\norm{x^i}_2\right)
\label{eq_kernelDL_l2}
\ee
Note that $G$ is not affected by this
as the OC-SVM problem deals with the sparse representations in $X$.

For the $\ell_{2,1}$ regularization from Section~\ref{sec:icassp_ocsvm}
we take into consideration only the $(a,x)$ pair
\begin{align*}
\frac12&\norm{\varphi(Y)(R - ax)}_F^2 + \beta \norm{x}_2 =
  \frac12 Tr[(R - ax)^T\varphi(Y)^T\varphi(Y)(R-ax)] + \beta \norm{x}_2 \\
  &= \frac12 Tr(R^TKR - R^TKax -x^Ta^TKR + x^Ta^TKax) + \beta \norm{x}_2 \\
  &= \frac12 (Tr(R^TKR) - 2xR^TKa + a^TKa\norm{x}_2^2) + \beta \norm{x}_2
\end{align*}
which, together with the objective from $G$,
leads to the modified K-SVD iteration
\be
(a^+_i,(x^i)^+) = \underset{a_i: \norm{K^{\frac12}a_i}=1,x^i}{\arg\min}
\frac12 \norm{x^i}^2 -x^i(R^TKa_i + \nu\lambda) + \beta\norm{x^i}
\label{inner_ax_kernelDL}
\ee
where we used the fact that $a_i^TKa_i = 1$.
We show how to solve this iteration
as a Corollary to Theorem~\ref{prop:icassp_ocsvm_l1}.

\begin{corollary}\label{prop:icassp_ocsvm_kernel}
Let $ \phi(z) = z $, $\nu$ the corresponding element of $\omega$ and K the kernel matrix. Then the closed form solution of the K-SVD iteration for \eqref{inner_ax_kernelDL} is:
if $\norm{R^TKa + \nu\lambda} \ge \beta$ then
\begin{align}
    a^*= K^{-\frac12} \arg\max\limits_{\norm{f} = 1} \; \frac{1}{2}\norm{R^TK^\frac12f + \nu\lambda}^2_2 \label{opt_icassp_ocsvm_kernel_l1_d}, \\
	x^*(a^*) = 
    \left(1 - \frac{\beta}{\norm{R^TKa^* + \nu\lambda}} \right)(R^TKa^* + \nu\lambda)^T\label{opt_icassp_ocsvm_kernel_l1_x},
\end{align}
otherwise $(d,x) = (d, 0)$.
Where solving for $a^*$ reduces to a trust region problem.
\end{corollary}

\begin{proof}
Denoting $x^*(a)$ the solution in $x$ for a fixed $a$, the first order optimality of
the inner problem reduces implies:
\begin{align}
0 & \in x^*(a)^T - R^TKa -\nu \lambda + \beta \partial \norm{x^*(a)}
\label{eq_kernelDL_xa_optimality}
\end{align}
Notice that if $\norm{R^TKa + \nu\lambda} \le \beta$ then $x^*(a) = 0$.
The proof follows that of Theorem~\ref{prop:icassp_ocsvm_l1}
leading to \eqref{opt_icassp_ocsvm_kernel_l1_x} and
\begin{align}\label{d_kernel_optimal}
	a^* = \arg\max\limits_{\norm{K^\frac12a} = 1} \; \frac{1}{2}\norm{R^TKa + \nu\lambda}^2_2
\end{align}
which can be rewritten by using $f = K^\frac12a$ as:
\begin{align}\label{d_kernel_optimal_f}
	f^* = \arg\max\limits_{\norm{f} = 1} \; \frac{1}{2}\norm{R^TK^\frac12f + \nu\lambda}^2_2
\end{align}
and then we can recover $a^*$ from $K^\frac12a^* = f^*$.

Given the eigendecomposition of the PSD matrix $K = U \Lambda U^T$ where $\Lambda$ is diagonal, then $K^\frac12 = U\Lambda^\frac12U^T$ and $a^*=U \Lambda^{-\frac12} U^T f^*$

\end{proof}

To conclude,
we sum-up the adaptations required in our algorithms for the kernel DL formulation.
As discussed,
Algorithm~\ref{alg:ksvd-ocsvm-training} step 1 will use the error $E=I-AX$
which will also reflect in the other steps involving the residual $R$.
Also,
in step 4 the update will regard the $(a,x)$
pairs of Corollary~\ref{prop:icassp_ocsvm_kernel}.
When performing anomaly detection with a trained model,
we have to adapt the following steps of Algorithm~\ref{alg:ksvd-ocsvm-testing}:
step 1 will call $\text{OMP}(A^T\Tilde{K}, A^TKA, s)$,
step 2 will compute the error as $E=K^{-1}\Tilde{K} - A\Tilde{X}$
and
during trimming in step 5 we will zero the lines
if the corresponding condition from Corollary~\ref{prop:icassp_ocsvm_kernel} holds.

\subsection{Kernel DPL}
\label{sec:kernel-DPL}




The kernel version of the DPL problem requires using kernels for the sparse representations as well.
Let us factorize,
without loss of generality, $P = B Y^T$,
in a fashion similar to the factorization $D = Y A$ used in kernel DL.
In this way each row $p^i$ of $P$ is expressed as a linear combination of the signals in $Y$,
with the coefficients $b^i$ as the $i$-th row of the tall matrix $B$, $p^i = b^i Y^T$. 
The DPL formulation \eqref{icassp_DPL} becomes
\begin{align}
F(Y, D, B)  &= \frac12 \norm{Y - DBY^TY}_F^2 
  + \beta \sum_{i=1}^n \phi\left(\norm{b^i Y^T Y}_2\right)
  + \sum_{i=1}^n\alpha_i \norm{b^i Y^T Y}_1.
\label{eq_DPL_PBY}\end{align}
In the kernel DL setting, using kernels $\varphi(Y)$, this becomes
\begin{align}
F(Y, A, B)  
&= \frac12 \norm{\varphi(Y) (I - A B \varphi(Y)^T \varphi(Y))}_F^2 + \\
    &\qquad + \beta \sum_{i=1}^n \phi\left(\norm{b^i \varphi(Y)^T \varphi(Y)}_2\right) + \sum_{i=1}^n\alpha_i \norm{b^i \varphi(Y)^T \varphi(Y)}_1 \\
&= \frac12 \norm{\varphi(Y) (I - A B K) }_F^2 
+ \beta \sum_{i=1}^n \phi\left(\norm{b^i K}_2\right)
   + \sum_{i=1}^n\alpha_i \norm{b^i K}_1
   \label{eq_kernel_DPL_F}
\end{align}
Note that equation \eqref{eq_kernel_DPL_F} is the kernel version of \eqref{icassp_DPL}, 
with $PY$ replaced now by $BK$. 

When solving iteratively in K-SVD fashion,
the subproblem resembles a combination 
of \eqref{min_min_start} and \eqref{inner_ax_kernelDL},
with $x = bK$:
\be
(a^+_i,(x^i)^+) = \underset{a_i: \norm{K^{\frac12}a_i}=1,x^i: x^i= b^iK}{\arg\min}
   \; \frac{1}{2}\norm{x^i}^2 - x \left( R^T K a_i + \nu \lambda\right) + \beta\norm{x^i}+ \alpha\norm{x^i}_1
   \label{min_min_start_kernel}\\
\ee
Note that \eqref{min_min_start_kernel} is in fact identical to \eqref{min_min_start} with $d \leftarrow K^\frac{1}{2}a$, $x \leftarrow bK$,  $R \leftarrow K^\frac{1}{2}R$ and $Y \leftarrow K$, and thus the solution can be found with the same procedure.
For completeness, we formalize the result below.

\begin{corollary}\label{prop:DPL_kernel}
Let $K^\dagger$ denote the Moore-Penrose pseudoinverse of $K$, and denote $H_K = K K^\dagger = K^\dagger K$ (since $K$ is symmetric). 
Then the solution of the inner K-SVD sub-problem \eqref{min_min_start_kernel}
for the kernel DPL global objective \eqref{eq_kernel_DPL_F} is the following:
\begin{align}
    (d^+, \tau_1(d^+),\tau_2(d^+)) & = \min\limits_{\norm{d} = 1} \max\limits_{\overset{\norm{\tau_2}_{2}\le 1,}{\norm{\tau_1}_{\infty}\le 1}   }      - \frac{1}{2}\norm{R^T K^\frac{1}{2} d + \nu \lambda  - \alpha \tau_1 - \beta \tau_2}^2_{H_K} \label{opt_kerneldpl_ocsvm_l1_d} \\
    a^+ & = K^{-\frac{1}{2}} d^+ \\
    b^+ & = (R^T K^\frac{1}{2} d^+ + \nu \lambda - \beta \tau_2(d^+) - \alpha \tau_1(d^+))^T K^\dagger \label{opt_kerneldpl_ocsvm_l1_b} \\
\end{align}
As a consequence, the optimal $x$ is:
\be
  x^+ = b^+ K = (R^T K^\frac{1}{2} d^+ + \nu \lambda - \beta \tau_2(d^+) - \alpha \tau_1(d^+))^T H_K 
  \label{opt_kerneldpl_ocsvm_l1_x}
\ee
\end{corollary}
\begin{proof}
The inner K-SVD sub-problem \eqref{min_min_start_kernel} is identical to \eqref{min_min_start} with the following change of variables: $d \leftarrow K^\frac{1}{2}a$, $R \leftarrow K^\frac{1}{2}R$ and $Y \leftarrow K$.
Note that $H_K$ is symmetric.

Therefore the solution is \eqref{min_max_3var}, with the required change of variables.
\end{proof}

In practice, following K-SVD, each update considers only the signals which use the corresponding atom $a$ in their representation, thus $R$ and $K$ contain only the columns corresponding to these signals.

With regard to Algorithm~\ref{alg:ksvd-ocsvm-training} and Algorithm~\ref{alg:ksvd-ocsvm-testing}, 
using the kernel DPL variant requires all the previous modifications implied separately by both DPL and kernel DL variants, respectively. Thus, in Algorithm~\ref{alg:ksvd-ocsvm-training} the error is computed as $E = I - AX = I - ABK$, and in step 4 the individual $(d,x)$ updates are now based on the $(a,b)$ pair defined in Corollary \ref{prop:DPL_kernel}. In Algorithm~\ref{alg:ksvd-ocsvm-testing}, step 1 will call the kernel $\text{OMP}(A^T\Tilde{K}, A^TKA, s)$, steps 2 and 4 are no longer required, and in step 5 we zero the rows of $\Tilde{X}$ where $\|bK\|$ is (almost) null.

\section{Results}
\label{sec:results}


\begin{table*}[t]
\tabcolsep 5pt
\caption{Datasets from ODDS used in the experiments}
\label{tab:datasets}
\small
\scriptsize
\begin{center}
\bt{ c | c | c | c | c | c | c c|}
Dataset & No. features & No. samples &  No. outliers & Sparsity & $\beta$ DL & \multicolumn{2}{c|}{$\beta$, $\gamma$  DPL}\\
\hline
satellite & 36 & 6435 & 2036(32\%) & 6 & 0.5 & 0.02 & 0.07 \\
shuttle & 9 & 49097 & 3511(7\%) & 4 & 0.32 &  0.14 & 0.03 \\
pendigits & 16 & 6870 & 156(2.27\%) & 5 & 0.115 & 0.14 & 0.04\\
speech & 400 & 3686 & 61(1.65\%) & 10 & 0.05 & 0.3 & 0.14 \\
mnist & 100 & 7603 & 700(9.2\%) & 8 & 0.01 &  0.02 & 0.16\\
\hline
Dataset & No. features & No. samples &  No. outliers & Sparsity & $\beta$ KDL & \multicolumn{2}{c|}{$\beta$, $\gamma$ KDPL} \\
\hline
cardio & 21 & 1831 & 176(9.6\%) & 5 & 0.15 & 0.05 & 0.05  \\
glass & 9 & 214 & 9(4.2\%) & 3 & 0.05 & 0.35 & 0.45\\
thyroid & 6 & 3772 & 93(2.5\%) & 3 & 0.05 & 0.05 & 0.05  \\
vowels & 12 & 1456 & 50(3.4\%) & 4 & 0.5 & 0.05 & 0.35\\
wbc & 30 & 278 & 21(5.6\%) & 6 & 0.45 & 0.25 & 0.05\\
lympho & 18 & 148 & 6(4.1\%) & 4 & 0.3 & 0.15 & 0.05\\
\et
\end{center}
\end{table*}

In order to validate our algorithms
we performed a set of experiments using the popular Outlier Detection Datasets~\footnote{\url{https://odds.cs.stonybrook.edu/}}.
The number of features, samples, and outliers (with the associated contamination rate)
together with the corresponding sparsity used by our models across the experiments for each of these datasets
are presented in Table \ref{tab:datasets}. 

We set the maximum number of outer iterations $K=6$ across all the experiments. A grid-search over 20 randomly generated dictionaries was performed for each of our models with the purpose of finding the dictionary that, together with the other hyperparameters, produces the best Balanced Accuracy (BA), which is defined as the arithmetic mean of True Positive Rate (TPR) and True Negative Rate (TNR). TPR represents the proportion of actual anomalies that were correctly identified as outliers and, respectively, TNR the proportion of actual normal samples that were accurately identified as normal.

We compare the results obtained by our methods with
the state of the art AD methods 
One Class - Support Vector Machine (OC-SVM)~\citep{ocsvm},
Local Outlier Factor (LOF)~\citep{lof},
Isolation Forest~\citep{iforest} and
Autoencoders (AE)~\citep{autoencoder}
that were optimized through
an extensive grid-search across multiple kernels, metrics and
hyper-parameters. For AE we used symmetric encoders and decoders, each having 3 or 4 layers (depending on the number of features in the datasets) and the same activation function for all layers.
The optimizer, dropout rate and activation functions were optimized by performing grid-search for each dataset.
Experiments were implemented using  Scikit-learn 1.3.0 and Python 3.11.5 on an AMD Ryzen Threadripper PRO 3955WX.
Our algorithms implementation and their applications are public and available online at \url{https://github.com/iulianhiji/AD-DL-OCSVM}.

\begin{figure}
    \centering
     \includegraphics[width=1\linewidth]{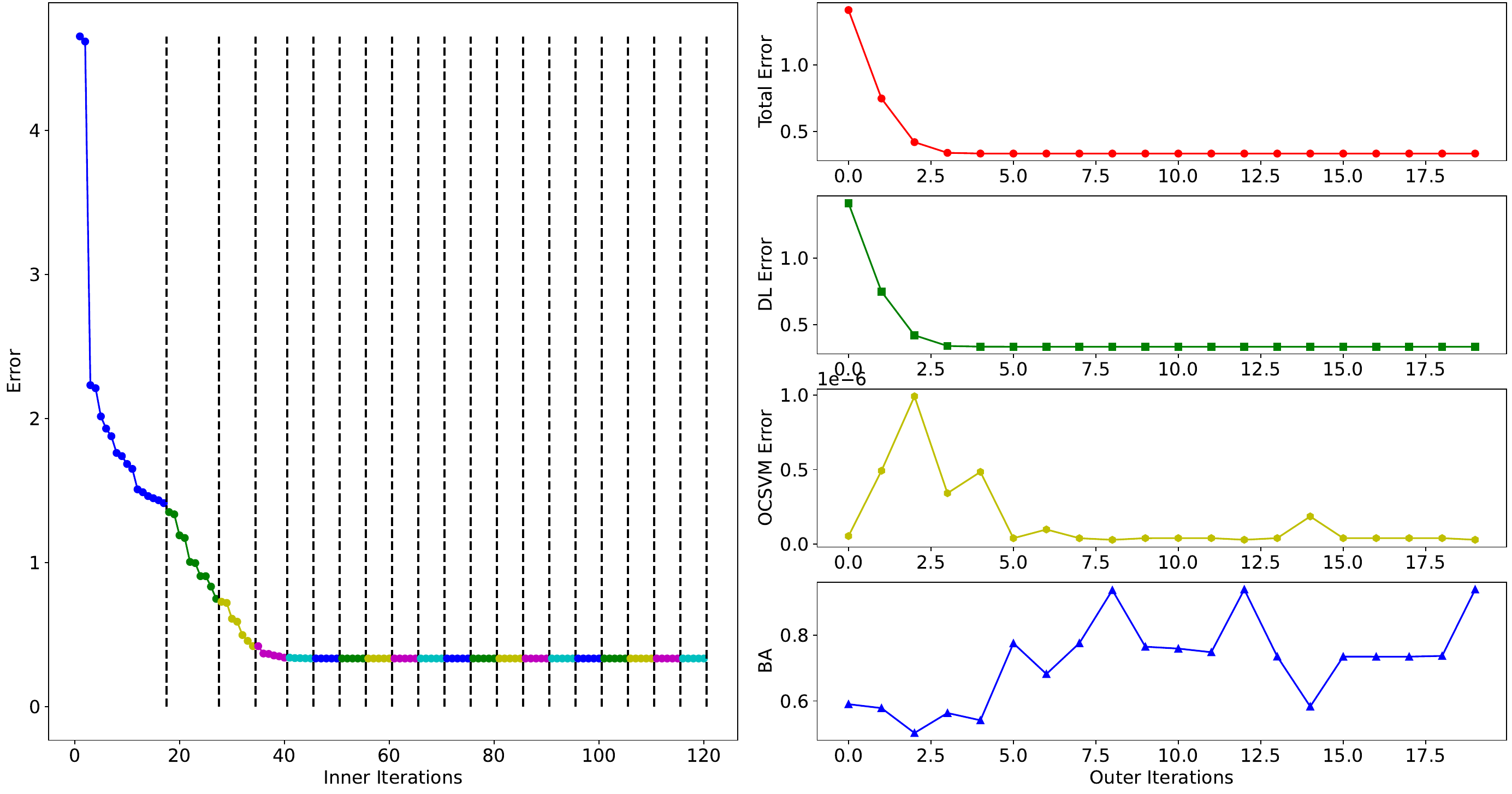}
    \caption{(Left) Convergence on shuttle dataset: each point is an inner iteration, vertical lines separate outer iterations.
    (Right) Outer iterations convergence analysis of total error $\mathcal{L}$ (first), DL error $F$ (second), OC-SVM error $G$ (third) and BA variation (fourth).}
    \label{fig:error_curves}
\end{figure}

We start by presenting the convergence of the inner iterations of our DL-OCSVM algorithm in Figure~\ref{fig:error_curves}(Left), where the dotted lines mark the end of an outer iteration. The total error computed after each inner iteration is strictly decreasing within each outer iteration. Also, after a number of only 4 or 5 outer iterations further computations do not seem to offer significant improvement.
In Figure~\ref{fig:error_curves}(Right) we can see how the total error, DL error (from \eqref{icassp}), OC-SVM error (from \eqref{ocsvm}) and BA vary across outer iterations. As we can see, even if the total error converges, the OC-SVM error is non-monotonic, replicating the effect on BA. This artifact appears  because we use an implementation of OC-SVM which does not support restarting. 

\begin{table*}[t]

\tabcolsep 5pt
\small

\caption{Max Balanced accuracy for all data}
\label{tab:newrealresults}
\begin{center}
\bt{c || c  | c  |
c  | c  | c  | c |}
Dataset
 & DL-OCSVM  & DPL-OCSVM & OC-SVM & LOF & IForest & AE \\
\hline
satellite
 & 0.7319 & \textbf{0.7760} & 0.6388 & 0.5677  & 0.7105 & 0.6543\\
shuttle
& 0.9400 & 0.9481 & 0.9512 & 0.5316  & 0.9768  &  \textbf{0.9782} \\
pendigits
& 0.8163 & 0.8724 & 0.8023 & 0.5895  & \textbf{0.8725} & 0.7669 \\
speech
 & 0.5283 & 0.5915 & \textbf{0.6830} & 0.5  & 0.5232 & 0.5246\\
mnist
& 0.8402 & \textbf{0.8528} & 0.7389 & 0.5736  & 0.7441 & 0.7320 \\
\et
\end{center}

\end{table*}



\begin{table*}[t]

\tabcolsep 3pt

\caption{TPR(left) and TNR(right) for all data}
\label{tab:newrealresults_tpr_tnr}
\scriptsize
\begin{center}
\bt{c || c c | c c |
c c | c c | c c | c c | }
\multirow{2}*{Dataset}
 & \multicolumn{2}{c|}{DL-OCSVM}  & \multicolumn{2}{c|}{DPL-OCSVM} & \multicolumn{2}{c|}{OC-SVM} 
& \multicolumn{2}{c|}{LOF} & \multicolumn{2}{c|}{IForest} 
& \multicolumn{2}{c|}{AE}\\
 & TPR & TNR & TPR & TNR & TPR & TNR & TPR & TNR & TPR & TNR & TPR & TNR \\
\hline
satellite
& 0.634 & 0.830 & \textbf{0.694} & \textbf{0.858} & 0.690 & 0.588
& 0.326 & 0.810  & 0.454 & 0.967  & 0.528 & 0.781 \\
shuttle
& 0.886 & 0.995 & 0.904 & 0.993 & 0.964  & 0.938 
& 0.130  & 0.933  & 0.983 & 0.971  & \textbf{0.960} & \textbf{0.996}  \\
pendigits
 & 0.949 & 0.684 & 0.968 & 0.777 & 0.987 & 0.618
& 0.263 & 0.916  & \textbf{1.0} & \textbf{0.745} & 0.545 & 0.989\\
speech
& 0.984 & 0.073 & 0.443 & 0.740 & \textbf{0.934} & \textbf{0.432}
& 0 & 1.0  & 0.049 & 0.997 & 0.066 & 0.984\\
mnist
& 0.917 & 0.763 & \textbf{0.869} & \textbf{0.837} & 0.934 & 0.544 
& 0.216 & 0.932  & 0.649 & 0.840 & 0.514 & 0.950 \\
\et
\end{center}

\end{table*}

\begin{table*}[t]

\tabcolsep 3pt

\caption{Training(left) and testing(right) time(in seconds) for all data}
\label{tab:train_test_time}
\scriptsize
\begin{center}
\bt{c || c c | c c |
c c | c c | c c | c c | }
Dataset
& \multicolumn{2}{c|}{DL-OCSVM}  & \multicolumn{2}{c|}{DPL-OCSVM} & \multicolumn{2}{c|}{OC-SVM} 
& \multicolumn{2}{c|}{LOF} & \multicolumn{2}{c|}{IForest}
 & \multicolumn{2}{c|}{AE} \\
\hline
satellite
 & 5.7 & 1.34 & 16.9 & 1.20 & 7.93 & 2.66
& 2.49 & -  & 0.73 & 0.11  & 25.69 & 0.31 \\
shuttle
& 100.1 & 7.1 & 126.3 & 7.1 & 101.06  & 46.47
& 449.43  & -  & 0.64  & 0.30 & 36.17 & 2.59  \\
pendigits
& 1.16 & 0.84 & 7.3 & 0.80 & 1.56 & 0.77
& 1.09 & -  & 0.90 & 0.11 & 23.77 & 0.40\\
speech
& 21.3 & 7.97 & 129.8 & 3.61 & 18.49 & 9.98
& 1.40 & -  & 0.45 & 0.18  & 3.39 & 0.21\\
mnist
& 31.1 & 3.78 & 52.3 & 2.88 & 21.78 & 6.97 
& 7.45 & -  & 0.53 & 0.18 & 31.21 & 0.39 \\
\et
\end{center}

\end{table*}

\begin{table*}[t]
\tabcolsep 5pt

\caption{Maximum Balanced accuracy obtained for all data on smaller datasets)}
\label{table:smallerd_db_results}

\small
\begin{center}
\bt{c || c | c |
c | c | c | c | c | c}
Dataset
& KDL-OCSVM & OC-SVM  &
LOF & IForest & AE\\
\hline
cardio
 & 0.7874 & 0.7896 & 0.5564  & \textbf{0.8620}  & 0.8617\\
glass
& \textbf{0.8951} & 0.7731  & 0.8566  & 0.7132 & 0.5360\\
thyroid
 & 0.7331 & 0.7565 & 0.5648  & \textbf{0.9506} & 0.7627\\
vowels
& \textbf{0.7869} & 0.7588 & 0.7419  & 0.6663  & 0.6582\\
wbc
 & \textbf{0.8991} & 0.7703 & 0.8137 & 0.8655 & 0.7969 \\
lympho
 & 0.9225 & 0.8309 & 0.8192 & 0.8873 &\textbf{0.9964} \\
\et
\end{center}

\end{table*}

We designed multiple experiments based on the two approaches discussed at the end of Section \ref{sec:methodology}. In the first experiment we used the entire dataset for both training and prediction in order to see how our models behave when they have to analyze a single batch of data. We trained the models described in Sections \ref{sec:icassp_ocsvm} and \ref{sec:icassp_DPL} using different values for $\beta$ and $\gamma$ (in the DPL model) in order to find the best ones;
the result parameters can be found in the last two columns of Table \ref{tab:datasets}.
The BA obtained by our models in this experiment are better than the ones of the competing models in 2 out of the 5 datasets and they closely follow the maximal ones in the other 3 cases, as it can be observed in Table \ref{tab:newrealresults}.
When compared to the proposed method from \citep[Table 1]{IRP22_ClipDL},
it can be easily observed that the inclusion of OC-SVM in the DL process has brought a clear overall improvement.

For further insights,
in Table \ref{tab:newrealresults_tpr_tnr} we show the TPR and TNR corresponding to results from Table \ref{tab:newrealresults}
while in Table \ref{tab:train_test_time} we present the training and testing times obtained while using the entire dataset for both.
The displayed times are generally longer than the ones corresponding to the competing models (with a few exceptions) due to the fact tha we perform the OC-SVM training in each outer iteration.
Table \ref{tab:train_test_time} does not show testing times for LOF because the scikit-learn implementation does not support prediction as a separate action.

In Table \ref{table:smallerd_db_results} we present the results obtained  with our Kernel DL-OCSVM method from Section \ref{sec:kernel-DL}.
Again, all the data were used both in training and testing.
We achieved the best results in 3 out of 6 datasets, proving that our kernel implementation manages to make use of the induced non-linearity in order to detect the outliers and were close to the best result in two other datasets.

\begin{figure}[t]
    \centering
    \includegraphics[width=1\linewidth]{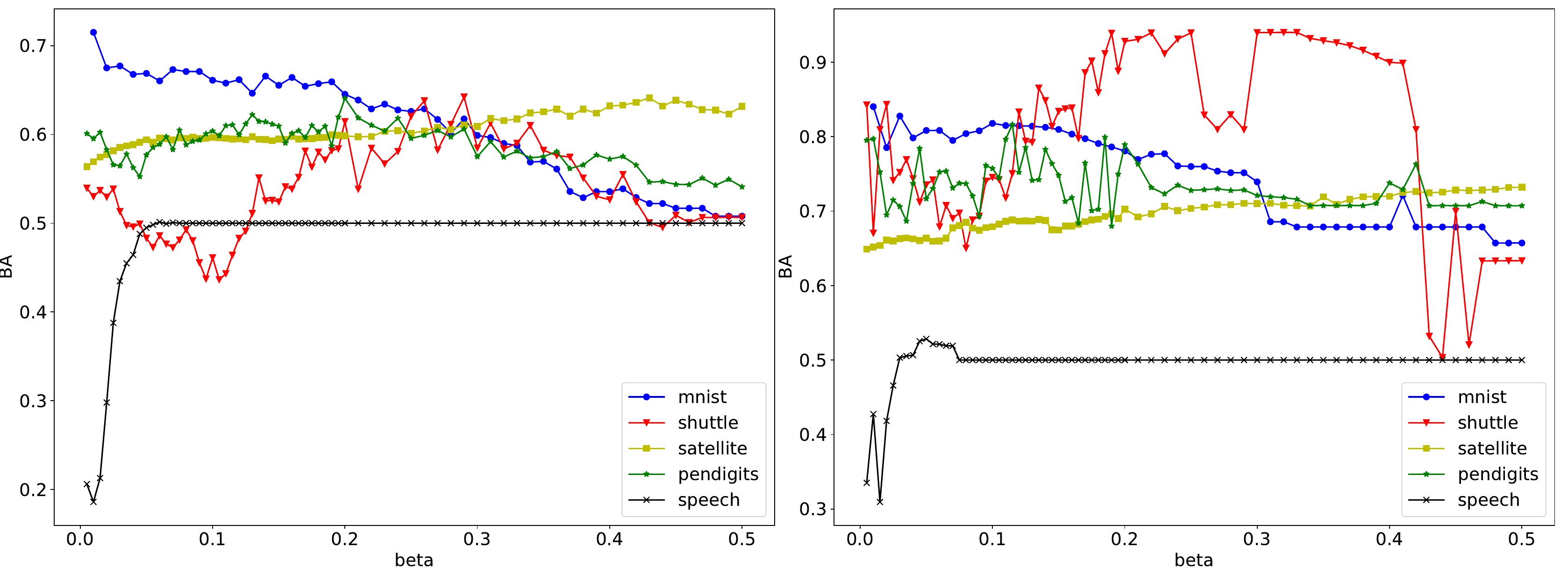}
   \caption{Mean (left) and maximum (right) BA for different values of $\beta$ on multiple datasets.}\label{Fig:BADL}
\end{figure}
\begin{figure}[t!]
  \centering
     \includegraphics[width=1\linewidth]{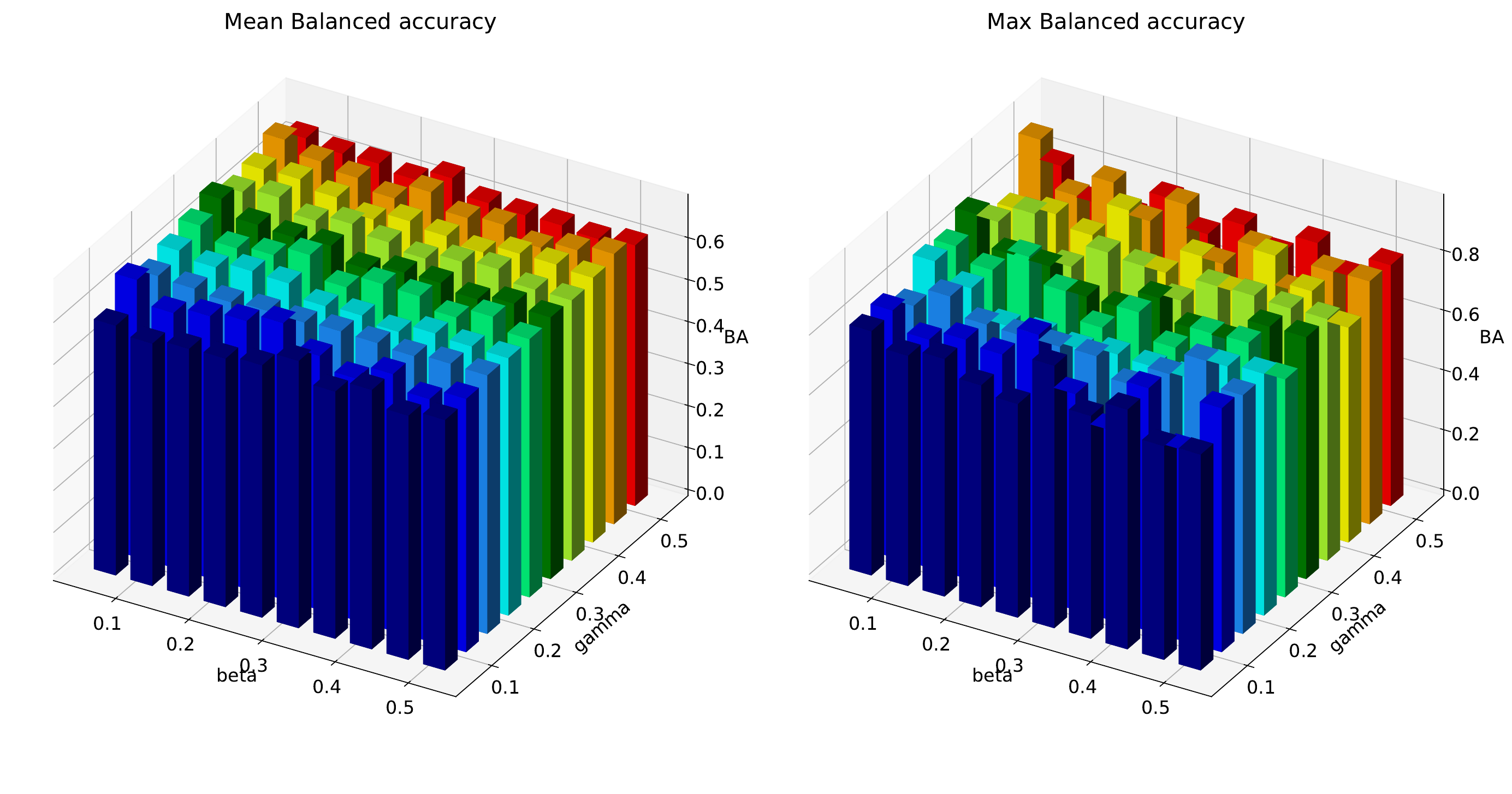}
   \caption{Grid-search for $\beta$ and $\gamma$ on lympho dataset
   with the resulting BA mean (left) and maximum (right).}\label{Fig:BADPL}
\end{figure}
Next we focus on the effect of the hyperparameters $\beta$ and $\gamma$ on the anomaly detection process.
In Figure \ref{Fig:BADL} we show the mean and maximum balanced accuracy obtained in the experiments for different values of $\beta$ while using the 20 randomly initialized dictionaries. For some of the datasets (like satellite, speech and pendigits) both mean and maximum BA have a lower variance, while for others (like mnist and shuttle) both metrics have a higher variance. The same metrics are analyzed for different values of $\beta$ and $\gamma$ from the DPL formulation from Section \ref{sec:icassp_DPL}. The values corresponding to the lympho dataset are displayed in Figure \ref{Fig:BADPL}.
We can see that we obtain better BA when both  $\beta$ and $\gamma$ have larger values.

\begin{figure}[t]
    \centering
     \includegraphics[width=1\linewidth]{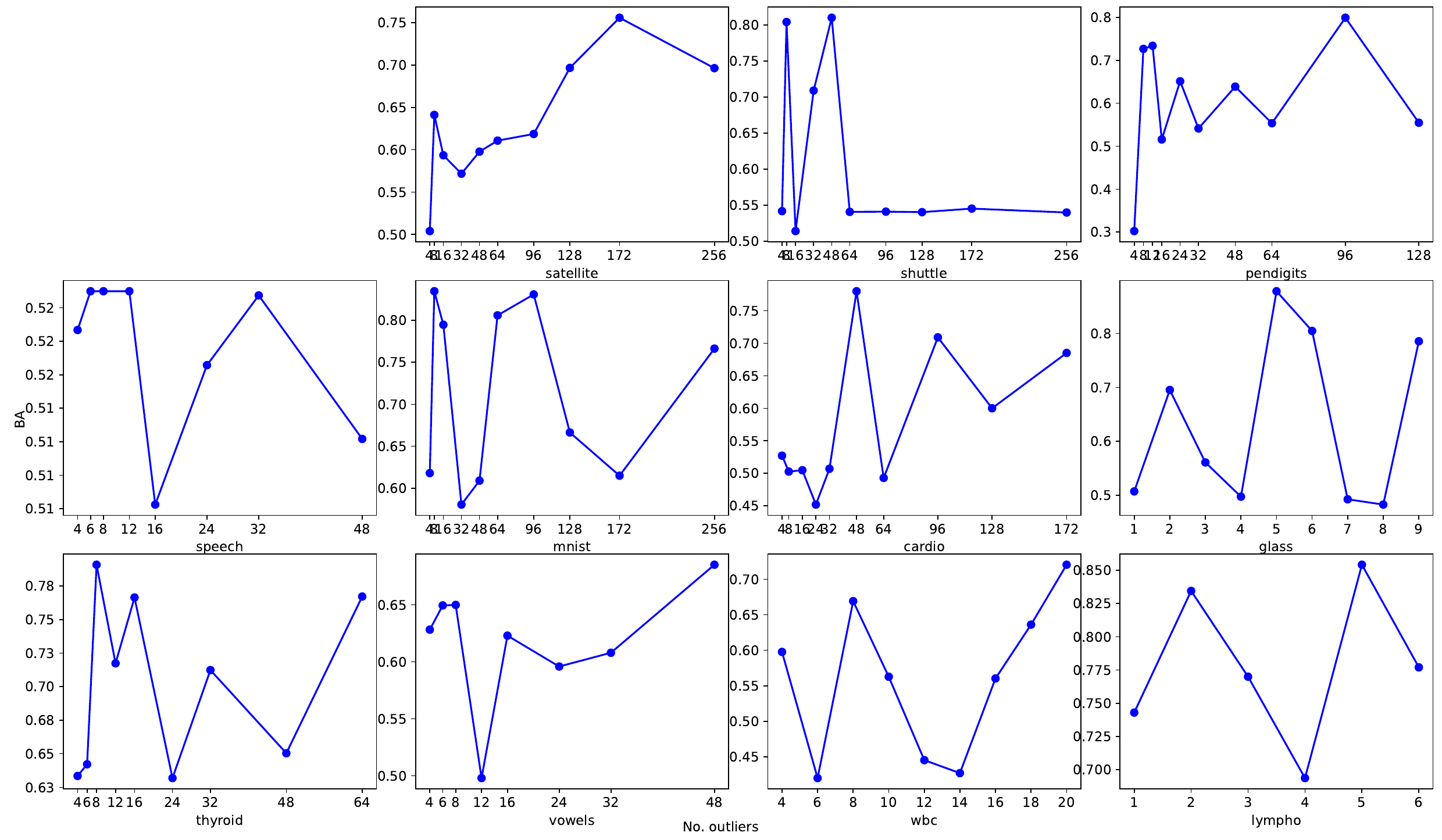}
    \caption{BA obtained in first scenario for different no. of outliers used for training + prediction}
    \label{fig:different_outliers_ba}
\end{figure}
The next hyperparameter to investigate is the contamination rate
and the robustness of our models when we vary the number of outliers found in the dataset.
The entire dataset is used for training.
We depict our findings in Figure \ref{fig:different_outliers_ba}
where we show the variation of BA obtained for all the datasets when we use all the inliers and an increasing number of the outliers for training. At the end we test the resulted representations.
The models are trained using the optimal DL parameters from Table \ref{tab:datasets}.
Because each time we start the training process from the beginning
with a different number of outliers 
we obtain variations in the balanced accuracy.
This happens because for each set of data a different dictionary,
a different set of representations and a different OC-SVM model are produced without a grid-search.
This has the side-benefit of depicting again the robustness to hyperparameter values as it can be seen that
the ranges of BA contain values that are near the maximal ones from Table \ref{tab:newrealresults}.
    
\begin{table*}[t]
\tabcolsep 5pt

\caption{Balanced accuracy obtained for testing data in the KFold CV scenario}
\label{tab:KFold_results}

\small
\begin{center}
\bt{c || c | c |
c | c | c | c | c | c }
Dataset
 & DL-OCSVM & OC-SVM  &
LOF & IForest & AE \\
\hline
satellite 
& 0.6354 & 0.6465 & 0.5764  & \textbf{0.6979} & 0.6272 \\
shuttle
& 0.5164 & 0.7878  & 0.5174  & \textbf{0.9769} & 0.8850 \\
pendigits
& \textbf{0.8041} & 0.6629 & 0.5533  & 0.5585 & 0.7436\\
speech
& \textbf{0.6196} & 0.4094 & 0.5452  & 0.4924 & 0.4917\\
mnist
& \textbf{0.7879} & 0.7182 & 0.6204 & 0.6186 & 0.6850 \\
\et
\end{center}

\end{table*}

In the last experiment,
we want to analyze how our models behave when they have to make predictions on a new dataset different from the one it was trained on.
In order to achieve this,
we designed an experiment that makes use of KFold cross-validation \citep{anguita2012k}.
The data are split into a validation set (80\%) and a testing set (20\%). The validation set is used by the KFold cross-validation procedure to select the best hyperparameters
(including one of the 20 randomly generated dictionaries)
and then the model is trained on the entire validation set
using the found hyperparameters.
The reported results are those obtained by making predictions on the testing set.
As we can see in Table \ref{tab:KFold_results} we outperformed the competing models in 3 out of 5 datasets and come in close with the others.

\section{Conclusions and future work}
\label{sec:conclusions}

In this paper we proposed a new unsupervised anomaly detection framework,
described in Algorithm~\ref{alg:general_scheme},
based on the composite objective of
uniform support dictionary learning 
and
one class support vector machines. 
We focused on modified K-SVD iterations
that fuse the uniform support approach with the OC-SVM objective
based on standard DL and the DPL formulation
for which we provided
Theorems~\ref{prop:icassp_ocsvm_l1} and \ref{prop:dpl_ocsvm_l1}
showing the optimal iteration steps.
Plugging the modified K-SVD iterations in the anomaly detection framework
produced Algorithms~\ref{alg:ksvd-ocsvm-training} and \ref{alg:ksvd-ocsvm-testing}
for model training and, respectively, the anomaly detection routine of new data points.
The convergence of the training iterations was analyzed in Proposition~\ref{prop:fixed_point}.
Furthermore,
we studied in Section~\ref{sec:kernel}
the extensions of the existing algorithms and their associated theorems
to kernel methods.

For all the proposed algorithms and their extensions
we provided numerical analysis and experiments
for the models parameters
together with performance comparison against
standard OC-SVM and uniform support DL
and also against Local Outlier Factor and Isolation Forest methods which are often found in the field.
Of note is the comparisons with deep autoencoders which showed
that our methods provide similar or better anomaly detection.

The proposed model is part of the category of algorithms whose output largely depends on the initialization point, so it is important to mention that the reported results are influenced by the selected random dictionaries.

In the future
we plan to extend our work
to other popular methods such as
Support Vector Data Description,
Lightweight Online Detector of Anomalies
and
Autoencoders.
\bibliographystyle{elsarticle-harv}
\bibliography{bib}
\appendix


\section{Algorithms}
\label{appendix:algo}

Algorithms \ref{alg:ksvd-ocsvm-training} and \ref{alg:ksvd-ocsvm-testing}
refer to the standard DL setting. For completeness, we provide below
the corresponding algorithms for the DPL, kernel DL and kernel DPL alternatives.
For the DPL formulation described in Section \ref{sec:icassp_DPL}, 
the Algorithms are \ref{alg:ksvd-ocsvm-training-DPL} and \ref{alg:ksvd-ocsvm-testing-DPL}. For the kernel DL formulation in Section \ref{sec:kernel-DL}, Algorithms \ref{alg:ksvd-ocsvm-training-kerDL} and \ref{alg:ksvd-ocsvm-testing-kerDL}. For the kernel DPL in section \ref{sec:kernel-DPL}, Algorithms \ref{alg:ksvd-ocsvm-training-kerDPL} and \ref{alg:ksvd-ocsvm-testing-kerDPL}.

\begin{algorithm}
\DontPrintSemicolon
\SetKwComment{Comment}{}{}
\KwData{train set $Y \in \rset^{m \times N}$,
$D^{k-1} \in \rset^{m \times n}$,
$P^{k-1} \in \rset^{n \times m}$,\\
\hspace{13mm}inner iteration k
}
\KwResult{dictionary $D^k$ and projection matrix $P^k$}
\BlankLine

Error: $E^k = Y - D^{k-1} P^{k-1} Y$ \\
\For{$i \in \{1,\dots,n\}$}{
    Atom error: $R^k = E^k + d_{i}^{k-1} p^{i,k-1} Y$ \\
    Update: new $(d_{i}^{k}, p^{i,k})$ according to Th.
    \ref{prop:dpl_ocsvm_l1}\\
    New error: $E^{k} = R^k - d_{i}^{k} p^{i,k} Y$ \\
}
\caption{DPL-OCSVM Uniform Representation Learning}
\label{alg:ksvd-ocsvm-training-DPL}
\end{algorithm}
\begin{algorithm}
\DontPrintSemicolon
\SetKwComment{Comment}{}{}
\KwData{test set $\Tilde{Y} \in \rset^{m \times \Tilde{N}}$,
dictionary $D \in \rset^{m \times n}$, \\
\hspace{13mm}projection matrix $P \in \rset^{n \times m}$,
sparsity $s$, tolerance $tol$
 and \\
\hspace{13mm}OC-SVM model $(\omega, \rho, \lambda)$
}
\KwResult{anomalies $\Tilde{\A}$}
\BlankLine
Representation: $\Tilde{X} = \text{OMP}(\Tilde{Y}, D, s)$ \\
\For{$i \in \{1, \dots, n\}$}{
    Trimming: zero $\Tilde{x}^{i}$ if $\| p^i\Tilde{Y} \| < tol$\\
}
\lIf{$\sign \left( \omega^T \Tilde{x}_i - \rho \right) \le 0$}{$\Tilde{\A} = \Tilde{\A} \cup \{i\} \ \;\ \forall i \in \Tilde{N}$}
\caption{DPL-OCSVM Anomaly Detection}
\label{alg:ksvd-ocsvm-testing-DPL}
\end{algorithm}

\begin{algorithm}
\DontPrintSemicolon
\SetKwComment{Comment}{}{}
\KwData{Gram matrix $K \in \rset^{N \times N}$,
$A^{k-1} \in \rset^{N \times n}$,
inner iteration k
}
\KwResult{dictionary $A^k$ and representations $X^k$}
\BlankLine

Error: $E^k = I - A^{k-1} X^{k-1}$ \\
\For{$i \in \{1,\dots,n\}$}{
    Atom error: $R^k = E^k + a_{i}^{k-1} x^{i,k-1}$ \\
    Update: new $(a_{i}^{k}, x^{i,k})$ according to Corollary \ref{prop:icassp_ocsvm_kernel}\\
    New error: $E^{k} = R^k - a_{i}^{k} x^{i,k}$ \\
}
\caption{KDL-OCSVM Uniform Representation Learning}
\label{alg:ksvd-ocsvm-training-kerDL}
\end{algorithm}
\begin{algorithm}
\DontPrintSemicolon
\SetKwComment{Comment}{}{}
\KwData{Gram matrices $\tilde{K} \in \rset^{N \times \tilde{N}}$
and $K \in \rset^{N \times N}$,
dictionary $A \in \rset^{N \times n}$, sparsity $s$,
 and OC-SVM model $(\omega, \rho, \lambda)$ \\
}
\KwResult{anomalies $\Tilde{\A}$}
\BlankLine
Representation: $\Tilde{X} = \text{OMP}(A^T\Tilde{K}, A^T K A, s)$ \\
Error: $E = K^{-1}\Tilde{K} - A \Tilde{X}$ \\
\For{$i \in \{1, \dots, n\}$}{
    Atom error: $R = E + d_i \Tilde{x}^{i}$ \\
    Trimming: zero $\Tilde{x}^{i}$ if condition from \eqref{ksvd_supp_iter} holds\\
}
\lIf{$\sign \left( \omega^T \Tilde{x}_i - \rho \right) \le 0$}{$\Tilde{\A} = \Tilde{\A} \cup \{i\} \ \;\ \forall i \in \Tilde{N}$}

\caption{KDL-OCSVM Anomaly Detection}
\label{alg:ksvd-ocsvm-testing-kerDL}
\end{algorithm}

\begin{algorithm}
\DontPrintSemicolon
\SetKwComment{Comment}{}{}
\KwData{Gram matrix $K \in \rset^{N \times N}$,
$A^{k-1} \in \rset^{N \times n}$,
inner iteration k
}
\KwResult{dictionary $A^k$ and projection matrix $B^k$}
\BlankLine

Error: $E^k = I - A^{k-1} B^{k-1} K$ \\
\For{$i \in \{1,\dots,n\}$}{
    Atom error: $R^k = E^k + a_{i}^{k-1} b^{i,k-1} K$ \\
    Update: new $(a_{i}^{k}, b^{i,k})$ according to Corollary \ref{prop:DPL_kernel}\\
    New error: $E^{k} = R^k - a_{i}^{k} b^{i,k} K$ \\
}
\caption{KDPL-OCSVM Uniform Representation Learning}
\label{alg:ksvd-ocsvm-training-kerDPL}
\end{algorithm}
\begin{algorithm}
\DontPrintSemicolon
\SetKwComment{Comment}{}{}
\KwData{Gram matrices $\tilde{K} \in \rset^{N \times \tilde{N}}$
and $K \in \rset^{N \times N}$,
dictionarie $A \in \rset^{N \times n}$, projection matrix $B \in \rset^{n\times N}$, sparsity $s$, tolerance $tol$
 and OC-SVM model $(\omega, \rho, \lambda)$ \\
}
\KwResult{anomalies $\Tilde{\A}$}
\BlankLine
Representation: $\Tilde{X} = \text{OMP}(A^T\Tilde{K}, A^T K A, s)$ \\
\For{$i \in \{1, \dots, n\}$}{
    Trimming: zero $\Tilde{x}^{i}$ if row $\|b^i\Tilde{K}\| < tol$ is null\\
}
\lIf{$\sign \left( \omega^T \Tilde{x}_i - \rho \right) \le 0$}{$\Tilde{\A} = \Tilde{\A} \cup \{i\} \ \;\ \forall i \in \Tilde{N}$}

\caption{KDPL-OCSVM Anomaly Detection}
\label{alg:ksvd-ocsvm-testing-kerDPL}
\end{algorithm}

\end{document}